\documentclass[twoside]{article}
\usepackage[accepted]{aistats2025}

\usepackage[round]{natbib}

\usepackage{url}            
\usepackage{booktabs}       
\usepackage{tablefootnote} 
\usepackage{amsfonts}       
\usepackage{nicefrac}       
\usepackage{microtype}      
\usepackage{xcolor}         
\usepackage{graphicx}

\usepackage{tcolorbox}

\usepackage{amsmath,amsthm}
\newtheorem{theorem}{Theorem}

\newtheorem{remark}{Remark}

\newtheorem{principle}{Principle}

\usepackage{pifont}
\definecolor{dodgerblue}{rgb}{0.12, 0.56, 1.0}
\newcommand{\cmark}{{\color{dodgerblue}\ding{51}}}%
\newcommand{\xmark}{{\color{red}\ding{55}}}%

\newcommand{\gmark}{{\color{orange}\mathrm{dep.}\ \Gamma}}
\usepackage{placeins} 

\usepackage{enumitem}
\usepackage{tabularx,booktabs} 
\usepackage{makecell}
\usepackage{multirow, makecell}
\usepackage{url}

\usepackage[ruled]{algorithm2e}
\usepackage{setspace}
\usepackage{hyperref}

\newcommand{\ie}{\textit{i.e.}, }
\newcommand{\eg}{\textit{e.g.}, }
\newcommand{\R}{\mathbb{R}}
\newcommand{\N}{\mathbb{N}}
\newcommand{\F}{\mathfrak{F}} 
\newcommand{\M}{\mathcal{M}} 
\newcommand{\A}{\mathcal{A}} 
\newcommand{\C}{\mathcal{C}} 
\newcommand{\U}{\mathcal{U}} 
\renewcommand{\S}{\mathcal{S}} 

\newcommand{\T}{\mathcal{T}} 
\newcommand{\nf}{n_{i}}

\newcommand{\hx}{\hat{x}}
\newcommand{\hf}{\hat{f}}

\newcommand{\gf}{\nabla f}
\newcommand{\Hf}{\nabla^2 f}
\newcommand{\dg}[1]{\Delta g_{#1}}
\newcommand{\norm}[1]{\left\Vert #1 \right\Vert}

\begin{document}

\twocolumn[

\aistatstitle{From Learning to Optimize to Learning Optimization Algorithms}

\aistatsauthor{Camille Castera\And Peter Ochs}

\aistatsaddress{University of Bordeaux\\ Bordeaux INP, CNRS, IMB, UMR 5251 \\ Talence, France\And Department of Mathematics and Computer Science\\Saarland University\\ Saarbrücken, Germany}
]

\begin{abstract}
  Towards designing learned optimization algorithms that are usable beyond their training setting, we identify key principles that classical algorithms obey, but have up to now, not been used for Learning to Optimize (L2O).
  Following these principles, we provide a general design pipeline, taking into account data, architecture and learning strategy, and thereby enabling a synergy between classical optimization and L2O, resulting in a philosophy of \emph{Learning Optimization Algorithms}.
  As a consequence our learned algorithms perform well far beyond problems from the training distribution.
  We demonstrate the success of these novel principles by designing a new learning-enhanced BFGS algorithm and provide numerical experiments evidencing its adaptation to many settings at test time.
\end{abstract}

\section{INTRODUCTION}

Learning to Optimize (\textbf{L2O}) is a modern and promising approach towards designing optimization algorithms that reach a new level of efficiency. L2O is even the state-of-the-art approach in some applications \citep{liu2019alista,zhang2021plug}. However, it mostly excels when designed and trained specifically for each application and
 still fails to be widely applicable without retraining models.
This need to adapt L2O specifically to each task is especially problematic given how difficult designing L2O algorithm is: the design is prone to many conceptional pitfalls and training models is not only computationally expensive \citep{chen2022learning} but also notoriously hard in L2O \citep{metz2019understanding}.
In contrast, standard optimization methods are widely applicable, sometimes way beyond the setting they were originally designed for, as attested for example by the success of momentum methods \citep{polyak1964some} in deep learning \citep{jelassi2022towards}. This transfer of performance to different classes of problems is often achieved only at the cost of tuning a few scalar hyper-parameters. Analytically designed optimization algorithms usually come with theoretical guarantees, which most L2O algorithms lack completely of.

To bring L2O algorithms closer to actual Learned Optimization Algorithms (\textbf{LOA}),
we identify key theoretical principles that hand-crafted optimization algorithms follow and provide strategies ensuring that L2O approaches inherit these properties.
Thereby we systematically unify the advantages of both worlds: flexible applicability and theoretically controlled convergence guarantees from mathematical optimization, and complex operations from machine learning beyond what can be analytically designed.
We illustrate our new approach by designing a learning-enhanced BFGS method. We present numerical and theoretical evidence that our approach benefits L2O.

\section{RELATED WORK}\label{ref::relatedwork}

\paragraph{Learning to Optimize} L2O \citep{li2016learning} is an active topic of research. A lot of work focuses on unrolling \citep{gregor2010learning,ablin2019learning,huang2022unrolled,liu2019alista} and ``Plug and Play'' \citep{venkatakrishnan2013plug,meinhardt2017learning,zhang2021plug,Terris2023} approaches which improve over hand-crafted algorithms in several practical cases. 
L2O often lacks theoretical guarantees, with a few exceptions where convergence is enforced via ``safeguards'' that restrict the method \citep{moeller2019controlling,heaton2023safeguarded,martin2024learning}, or estimated statistically in distribution \citep{sucker2023pac}. One can also learn ``good initializations'' before using convergent algorithms \citep{sambharya2024warmstart}.
The L2O literature is broad, see \cite{chen2022learning,amos2023tutorial} for detailed overviews of the topic.

\paragraph{Design principles} One of our contributions (see Section~\ref{sec::principles}) is to enforce robustness to geometric transformations in L2O. This is related to ``geometric deep learning'' \citep{bronstein2021geometric}: the more general topic of preserving equivariance properties (defined in Section~\ref{sec::principles}) in the context of learning. It has many applications \citep{Romero2020,Chen2021,Hutchinson2021,Terris2023,Chen2023,keriven2024functions,levin2024any}.
It is also connected to learning on sets \citep{zaheer2017deep,lee2019set}.
To the best of our knowledge equivariance for L2O is only considered by \citet{ollivier2017information}, from a probabilistic point of view and \citet{tan2023boosting}. The latter argues that equivariances properties generally benefit L2O algorithms. Our work rather identifies specific key equivariance properties and studies how to handle them in every step of a general L2O pipeline (through Algorithm~\ref{algo::GenL2O}).
Parallel to our approach, \cite{liu2023towards} proposed to enforce convergence properties by design, hence stabilizing L2O methods, whereas we focus on generalization. Finally, improving the design of L2O algorithms from a practical perspective has been studied in \citep{wichrowska2017learned,metz2019understanding,metz2022practical}

\paragraph{Learning quasi-Newton methods}
We use learning to enhance a BFGS-like algorithm. BFGS \citep{broyden1970convergence,fletcher1970new,goldfarb1970family,shanno1970conditioning} is the most popular quasi-Newton (QN) algorithm and has been extensively analyzed \citep{greenstadt1970variations,dennis1977quasi,ren1983convergence}. Many extensions have been proposed, featuring limited memory \citep{liu1989limited}, sparse \citep{toint1981sparse} and non-smooth \citep{wang2022inertial} versions, or modifications provably faster in specific settings \citep{rodomanov2021greedy,jin2022sharpened}. Other approaches to make use of second-order derivatives only relying on gradients include  symmetric-rank-one methods \citep{conn1991convergence,NIPS2012_e034fb6b} and the dynamical inertial Newton family of methods \citep{alvarez2002second,attouch2016fast,attouch2022first} which is at the interface of first-and second-order optimization \citep{castera2024continuous}.

Several approaches, have previously been proposed to learn BFGS methods. A transformer model has been derived by \citep{gartner2023transformer}, and \citep{liao2023learning} considered learning on the fly in the online setting. A recent work \citep{lilearning} predicted a weighted average between DFP \citep{powell1983variable} and BFGS (a.k.a. a Broyden method). This is more akin to hyper-parameter tuning as their method remains in the span of Broyden's methods.
Our approach allows to build rather different learned QN algorithms by using a variational derivation of BFGS (see Section~\ref{sec::variational}), originating from \citep{greenstadt1970variations,goldfarb1970family} and which has been used in \citep{hennig2013quasi} for Bayesian optimization.

\section{SETTING AND PROBLEM STATEMENT}

We propose a mathematical formalism for L2O, akin to the ``semi-amortized'' framework from \citep{amos2023tutorial}. Compared to the latter we further decompose the algorithm into four pieces that will later allow us to mathematically discuss our main contribution of providing L2O algorithms with optimization properties in Section~\ref{sec::principles}.

\subsection{Mathematical Formalism\label{sec::mathformalism}}
We denote by $\N$ the set of non-negative integers and $\R$ the set of real numbers.
In what follows we consider unconstrained optimization problems of the form
\begin{equation}\label{eq::optimprob}
	\min_{x\in\R^n} f(x),
\end{equation}
where $n\in\N$ is the dimension of the problem, and $f$ belongs to $\F_n$: the set of real-valued lower bounded twice-continuously differentiable functions on $\R^n$ (with inner product $\langle\cdot,\cdot\rangle$ and norm $\norm{\cdot}$). Gradient and Hessian matrix of $f$ are denoted by $\gf$ and $\Hf$ respectively.

{We consider L2O models that are applicable in any dimension} (see Principle~\ref{princ::main} below), like standard optimization algorithms. Therefore in the sequel, the dimension $n$ is arbitrary and need not be the same for all the problems the algorithms are applied to. Nevertheless, for the sake of simplicity, the following discusses a fixed $\R^n$.
We call \emph{problem}, a triplet $(f,x_0,S_{0})$, made of an objective function $f\in\F_n$, an initialization $x_0\in\R^n$ and a collection of vectors and matrices $S_0\in \mathfrak{S}_n$, called \emph{state} ($\mathfrak{S}_n$ is the set of all possible states, clarified hereafter).

We formulate L2O algorithms in a generic form described in Algorithm~\ref{algo::GenL2O} that takes as input a problem $(f,x_0,S_0)$ and performs $K\in\N$ iterations before returning $x_K\in\R^n$.
Algorithm~\ref{algo::GenL2O} can also be mathematically represented by an operator $\A\colon\F_n\times\R^n\times\mathfrak{S}_n\times \N\to \R^n$ such that the $K$-th iteration of the algorithm reads
\[
x_{K} = \A\left(f,x_0,S_{0},K\right).
\]
Algorithm~\ref{algo::GenL2O} is fully characterized by what we call an \emph{oracle} $\C$, a \emph{model} $\M_\theta$,  an \emph{update} function $\U$ and a \emph{storage} function $\S$.
At any iteration $k\in\N$, the oracle $\C$ collects the information the algorithm has access to about $f$ at the current point $x_k$ and the state $S_k$ and constructs an \emph{input} $I_k\in \R^{n\times \nf}$, where $\nf\in\N$.
The input $I_k$ is then fed to a (machine learning) model represented by a parametric function $\M_\theta\colon \R^{n\times \nf}\to \R^{n\times m}$, where $\theta\in\R^p$ ($p\in\N$) is its parameter (in vector form), and $m\in\N$. The model outputs a \emph{prediction}, \ie $y_k = \M_\theta(I_k)$, which is used by the update $\U\colon \R^{n\times\nf}\times\R^{n\times m}\times \R^{n_h}$ to improve the current point: $x_{k+1} = x_k+ \U(I_k,y_k,\Gamma)$, where $\Gamma\in\R^{n_h}$ ($n_h\in \N$) are the hyper-parameters chosen by the user (a few scalars).
The storage $\S$ then collects, in $S_{k+1}$, the information from the $k$-th iterate that will be used at the next iteration.
This abstract formalism is generic enough to encompass at the same time L2O and several classical algorithms. Moreover, this systematic structuring allows for the formulation and analysis of key principles for LOA in Section~\ref{sec::principles}. We now illustrate this on an example.

\begin{center}
\newlength{\commentWidth}
\setlength{\commentWidth}{2.5cm}
\newcommand{\atcp}[1]{\tcp*[r]{\makebox[\commentWidth]{\small #1\hfill}}}
\begin{algorithm}
	\DontPrintSemicolon
	\SetKwFor{For}{for}{:}{}
	\SetKwInput{Input}{input}
	\SetKwInput{Given}{given}
	\Given{oracle $\C$, model $\M_\theta$, update $\U$, storage $\S$}
	\Input{problem ($f$, $x_0$, $S_0$), number of iterations $K$, hyper-parameter $\Gamma$}
	\setstretch{1.2}
	\For{$ k = 0 $ to $K-1$}{
		$ I_k \gets  \C(f,x_k,S_k)$ \atcp{Construct input}
		$ y_k \gets  \M_\theta(I_k)$ \atcp{Model prediction} 
		$ x_{k+1} \gets x_k+ \U(I_k,y_k,\Gamma)$ \atcp{Update step} 
    $ S_{k+1} \gets \S(S,x_k,I_k,y_k)$ 		\setlength{\commentWidth}{6cm}
\atcp{Store relevant variables in state}
	}
	\Return{$x_K$}
	\caption{Generic LOA\label{algo::GenL2O}}
\end{algorithm}
\end{center}

\paragraph{Example} 
Throughout what follows we use the heavy-ball (HB) method~\citep{polyak1964some} as running example to illustrate the concepts we introduce. An iteration $k\in\N$ of HB reads:
\begin{equation}\label{eq::HB}
	x_{k+1} = x_k + \alpha d_{k} - \gamma\gf(x_k),
\end{equation}
where $d_{k}=x_k-x_{k-1}$, $\alpha\in [0,1)$ is called the momentum parameter and $\gamma>0$ is the step-size. Notice that for $k=0$, the algorithm requires not only $x_0$ but also $x_{-1}\in\R^n$. This is the reason for introducing a state in Algorithm~\ref{algo::GenL2O}, in this case we would have $S_0=\{x_{-1}\}$. Any iteration $k$ of HB reads as follows:
	the operator $\C$ takes $(f,x_k,S_k)$, where $S_k = \{x_{k-1}\}$, and concatenates $\nabla f(x_k)$ and $d_{k}$ as $I_k=(d_{k},\gf (x_k))\in\R^{n\times \nf}$, with $\nf=2$.
	There is no learning phase hence no model $\M_\theta$ (by convention we say that $m=0$ and $y_k=0$).
	The update function $\U$ has hyper-parameter $\Gamma = (\alpha,\gamma)$--- so $n_h=2$ ---and uses $I_k$ to compute $\U(I_k,0,\Gamma) = \alpha{d_{k}} -\gamma \gf(x_k)$, which yields the update \eqref{eq::HB}. Finally, the storage $\S$ stores $x_k$ in $S_{k+1}$, as it will be required to compute $d_{k+1}$.

\subsection{{LOA is L2O with Specific Generalization}\label{sec::stagesL2O}}
{We detail the stages of building L2O models and the specific goal of a subfield of L2O: our LOA approach.}

{
\emph{Training}: The parameter $\theta$ of an L2O model $\M_\theta$ is set by minimizing a \emph{loss function} measuring how well \eqref{eq::optimprob} is solved on a \emph{training set} of problems $(f,x_0,S_0)$ (see \eqref{eq::lossfunction}). 
Training is crucial to find a ``good'' $\theta$ but computationally expensive and hard in L2O \citep{metz2019understanding}. Our goal is to study cases where the model is trained on a fixed training set, and then used on other functions $f\in\F_n$ \emph{without retraining}.

\emph{Test phase}: In machine learning, it standard to assume that the training set is sampled from an unknown underlying distribution of problems. Generalization, in the \emph{statistical sense},  refers to asserting how the trained model $\M_\theta$ performs on the whole distribution. This is  estimated by computing the performance on a \emph{test set} sampled independently from the same distribution.
}

{
\emph{Generalization}: One may wonder how the model $\M_\theta$ performs on problems not sampled from the aforementioned distribution. This is called out of distribution generalization and is unachievable in full generality \citep{wolpert1996lack}.
Instead, we note that hand-crafted optimization algorithms possess a different type of generalization property. We can sometimes use them for functions $f$ (and initializations) they were not designed for, only by tuning the hyper-parameter $\Gamma$. For example, HB was originally designed for locally $C^2$ functions~\citep{polyak1964some}, but works on larger classes of convex functions \citep{ghadimi2015global} and even performs well on non-convex ones \citep{zavriev1993heavy}. It also does not require a specific initialization and convergence rates are uniform in the dimension $n$ \citep{polyak1987introduction,bertsekas1997nonlinear}.
\emph{LOA differs from the rest of L2O} by identifying specific generalization properties that most hand-crafted optimization algorithms have, and designing L2O algorithms that features them.
We formulate these properties as a list of \emph{principles}.
}

\section{{THE PRINCIPLES OF LOA}}\label{sec::principles}

\begin{figure*}[t]
	\centering
	\begin{minipage}{.18\linewidth}
		\centering
		\includegraphics[width=\linewidth]{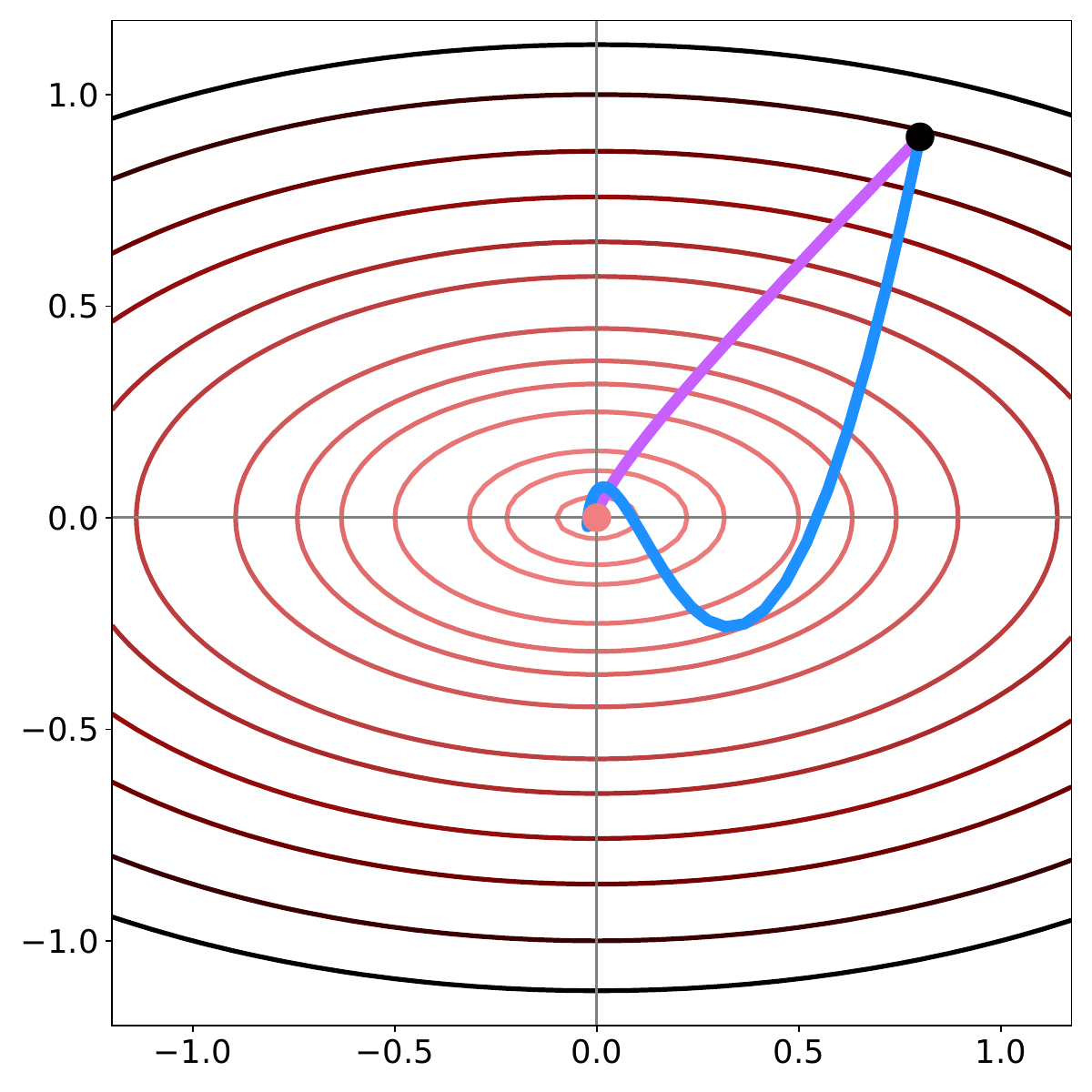}
	\end{minipage}
	\begin{minipage}{.18\linewidth}
		\centering
		\includegraphics[width=\linewidth]{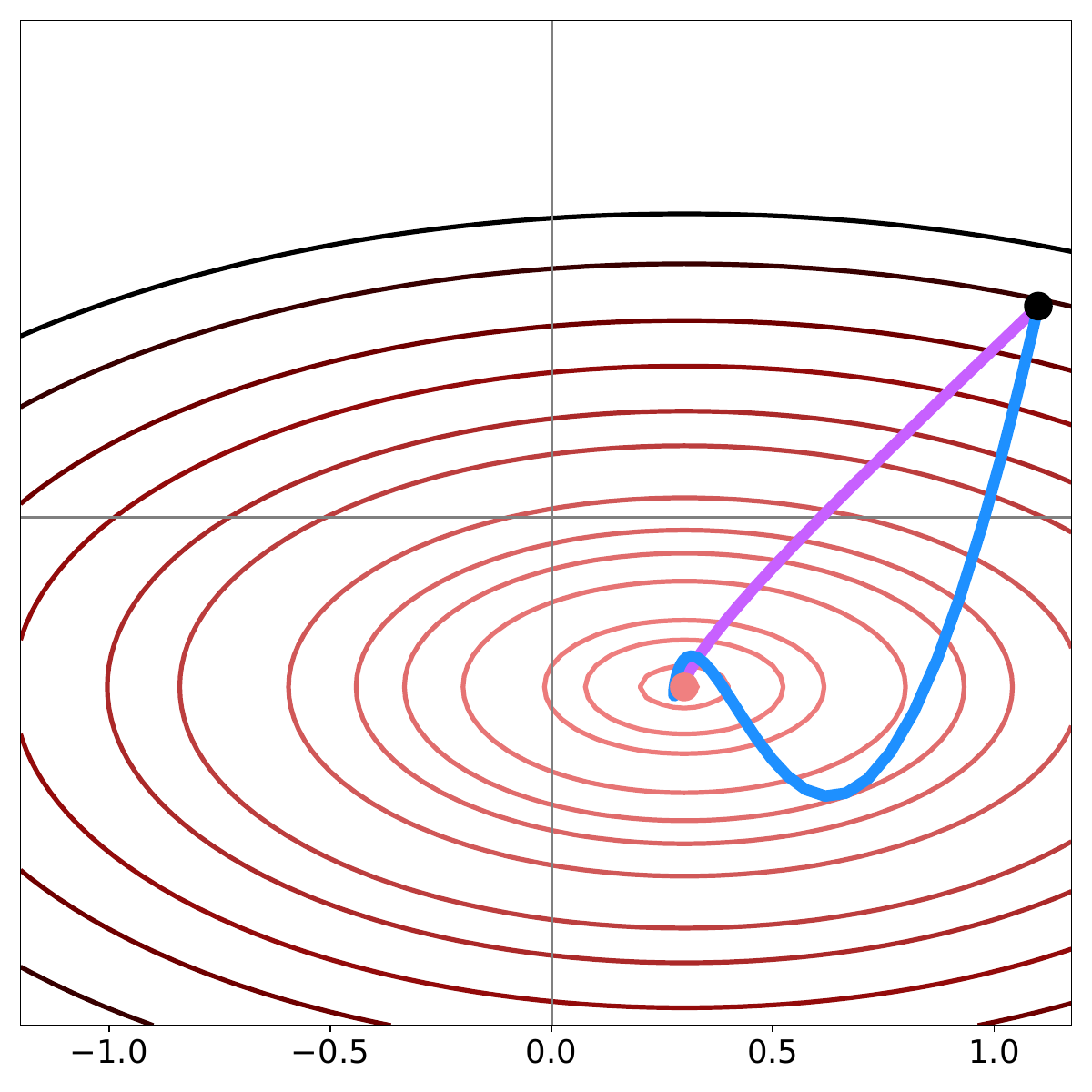}
	\end{minipage}
	\begin{minipage}{.18\linewidth}
		\centering
		\includegraphics[width=\linewidth]{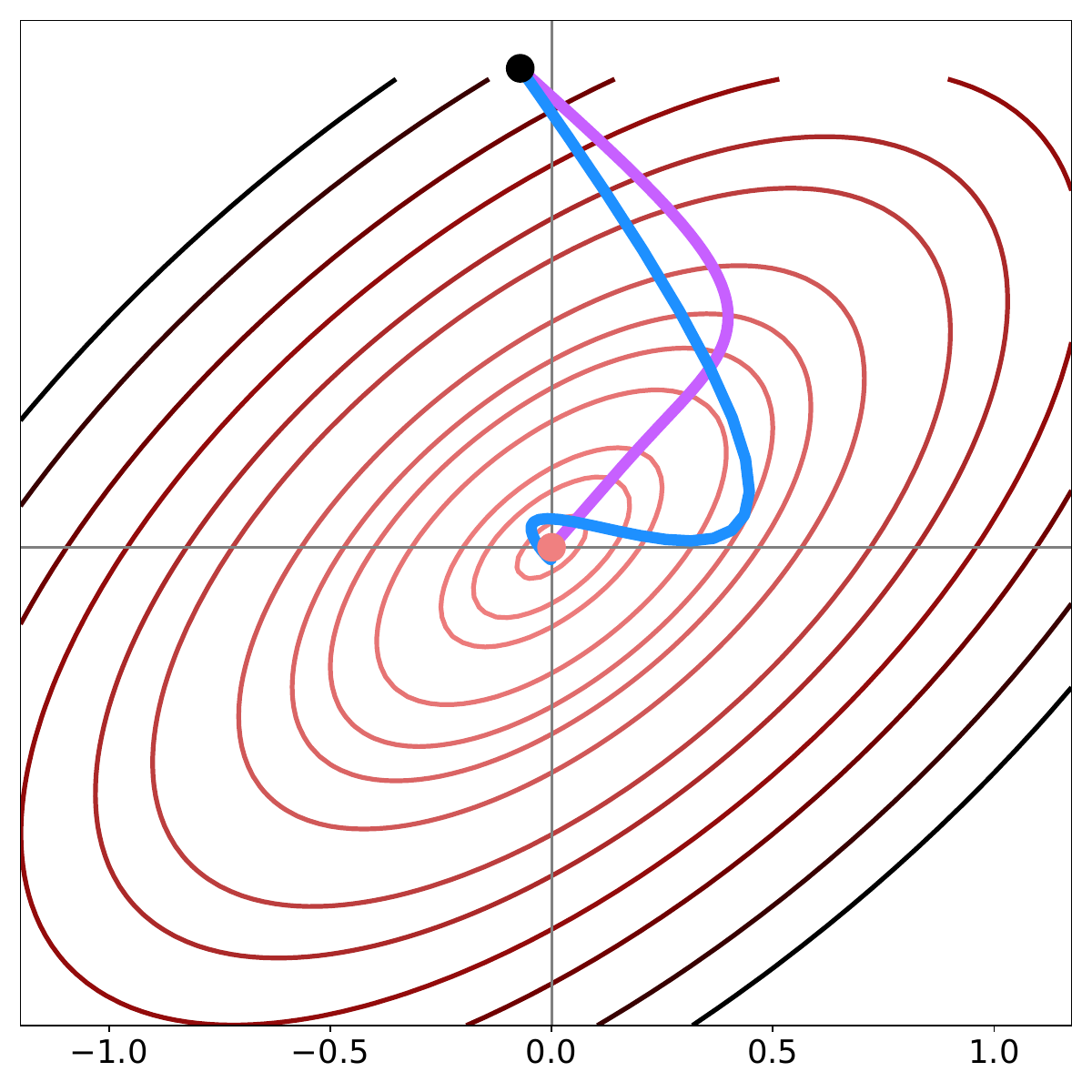}
	\end{minipage}
	\\
	\begin{minipage}{.55\linewidth}
		\centering
		\includegraphics[width=\linewidth]{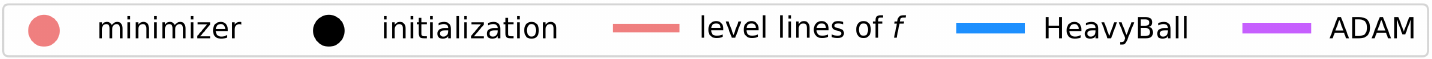}
	\end{minipage}
	\caption{Illustration of equivariances on the landscape of a 2D function. Left: no transformation, middle: translation, right: rotation. Transforming $f$ and $x_0$ does the same to the iterates of HB. ADAM \citep{kingma2014adam} is translation equivariant but not rotation equivariant.\label{fig::transformations}}
\end{figure*}

The cornerstone principle of LOA, is that optimization algorithms should be applicable in any dimension $n$. Since we \emph{do not retrain the model $\M_\theta$}, this implies the following.
\begin{principle}\label{princ::main}
	Algorithm~\ref{algo::GenL2O} should be independent of the dimension $n$, \ie the size $p$ of $\theta$ and $\nf$ should be independent of $n$ and as small as possible.
\end{principle}

{This principle makes LOA very different from the rest of L2O: the algorithm may be used on problems where $p$ is much smaller than $n$. LOA is thus in the \emph{under-parameterized} learning regime, so the training phase cannot be used to memorize many examples (unlike in the over-parameterized case \citep{zhang2021understanding}). We propose to cope with this via a careful algorithmic design revolving around three ideas.
}

\paragraph{Enhancement:} We use learning to enhance existing hand-crafted algorithms, preserving their theoretically-grounded parts. We only replace parts based on heuristics with learning, eventually reducing the size of $\theta$.

\paragraph{Adaption:} LOA must adapt \emph{on the fly} (along the iterates) to each problem by storing information in the state $S_k$. This can be achieved by recurrent neural networks, \textit{e.g.}, LSTMs \citep{andrychowicz2016learning,hochreiter1997long} or by enhancing adaptive algorithms like ADAM~\citep{kingma2014adam} or quasi-Newton methods, as we do in Section~\ref{sec::L2OQN}.

\paragraph{Hard-coded generalization:} {
	We show that most hand-crafted algorithms share generalization properties expressed through equivariance to key geometric transformations. This is one of the main contributions of our work, to which the rest of this section is dedicated to.
}
 Fix $f\in\F_n$, $x_0\in\R^n$ and $S_0$ and consider an invertible mapping $\T\colon\R^n\to\R^n$. Since $\T$ is invertible, observe that for all $x\in\R^n$:
\begin{equation}\label{eq::transformf}
  f(x) = f(\T^{-1}(\T(x))) = {f\circ\T^{-1}}({\T(x)}) = \hat{f}(\hat{x}),
\end{equation}
where we define $\hat{f}$ as $f\circ\T^{-1}$ and $\hat{x}=\T(x)$, for all $x\in\R^n$.
 Therefore \eqref{eq::transformf} expresses in particular that $(f,x_0,S_0)$ and $(\hat{f},\hat{x}_0,\hat{S}_0)$ are \emph{two different representations of the same problem}, where $\hat{S}_0$ is one-to-one with $S_0$ such that for every vector\footnote{The case of matrices contained in $S_0$ is more complex and discussed in Appendix~\ref{app:proofPrinc}.} $\hat{v}\in\hat{S}_0$ there exists $v\in S_0$ such that \ $\hat{v}=\T(v)$.
 {One would naturally want optimization algorithms (including Algorithm~\ref{algo::GenL2O}) to perform the same regardless the representation}, \ie
\begin{equation}\label{eq::invarianceoff}
	f(\A(f,x_0,S_0,K)) = \hat{f}(\A(\hat{f},\hat{x}_0,\hat{S}_{0},K)),\ \forall K\in\N.
\end{equation}
According to \eqref{eq::transformf}, a sufficient condition is that $\A(\hat{f},\hat{x}_0,\hat{S}_{0},K) =  \T\left(\A(f,x_0,S_0,K)\right)$ holds for all $(f,x_0, S_0)$ and $K$. When this is true, we say that the algorithm $\A$ is \emph{equivariant} to $\T$.

\paragraph{Link with Generalization}
Machine learning exploits similarity in data; in L2O, this means similarity between landscapes of objective functions. This is in line with our approach since \eqref{eq::transformf} expresses a specific form of similarity: that w.r.t.\ invertible transformations $\T$.
Although equivariances can sometimes be learned via data augmentation \citep{flinth2023optimization}, we argue that in the under-parameterized setting (where LOA belongs), enforcing those by design avoids wasting parts of the small parameter $\theta$ relearning them.

\paragraph{Trade-off}
While one would naturally want equivariance with respect to any invertible $\T$, this imposes severe restrictions on the design of Algorithm~\ref{algo::GenL2O}.
We therefore analyze equivariance only with respect to key transformations, summarized in Table~\ref{tab::equiv}.
Actually, even most hand-crafted algorithms do not achieve all the equivariances considered in Table~\ref{tab::equiv}.
A notable exception is Newton's method, which is unsuitable for large-scale optimization.
There is thus \emph{always a tradeoff to find}.
In fact, for hand-crafted algorithms, \citet{fletcher2000Newton} hypothesizes that the success of BFGS comes from the tradeoff it achieves between its computational cost and the equivariances it possesses.

We now discuss how to enforce specific equivariances  by design in Algorithm~\ref{algo::GenL2O}.

\begin{table*}[t]
	\caption{Summary of invariance and equivariance properties of several algorithms (proved in Appendix~\ref{app::priors})
	\label{tab::equiv}}
	\centering
	\begin{tabular}{lccccc}
		\toprule
		&\makecell{ Translation \\ (Principle~\ref{princ::translation})}&\makecell{ Permutation  \\ (Principle~\ref{princ::permut})} & \makecell{ Orthogonal \\ transform.}
		&
		\makecell{ Geom. scaling \\ (Principle~\ref{princ::geomscale})}
		&
		\makecell{ Func. scaling \\ (Principle~\ref{princ::scale})}
		\\
		\midrule
		Gradient desc. & \cmark &\cmark &\cmark & $\gmark$ &$\gmark$ \\
		Heavy-Ball & \cmark &\cmark &\cmark & $\gmark$ & $\gmark$ \\
		Newton Meth. & \cmark &\cmark &\cmark &\cmark &\cmark \\
		BFGS & \cmark &\cmark &\cmark &\cmark &\cmark \\
		ADAM  & \cmark &\cmark &\xmark &$\gmark$ &\cmark \\
		Algorithm~\ref{algo::L2OQN} & \cmark &\cmark &\xmark& \cmark & \cmark\\
		\bottomrule
	\end{tabular}
\end{table*}

\subsection{Translations\label{sec::translation}}
Let $v\in\R^n$, the translation $\T_v$ is defined for all $x\in\R^n$ by $\T_v(x)=x+v$. Then $\hat{f} = f(\cdot-v)$ and $\hat{x}_0=x_0+v$. Most algorithms are translation equivariant (see Table~\ref{tab::equiv} and Figure~\ref{fig::transformations}), which leads to the following principle.
\begin{principle}
	\label{princ::translation}
	Algorithm~\ref{algo::GenL2O} should be translation equivariant, \ie $\T_v$-equivariant for all  $v\in \R^n$. 
\end{principle}

\paragraph{Strategy}
Remark that since $\hat{x}_0 = x_0+v$, then $\hat{x}_K = x_K +v \iff \sum_{k=0}^{K-1} \hat{d}_k = \sum_{k=0}^{K-1} d_k$. So we want to ensure that $\forall k\in\N$, $d_k = \hat{d}_k$.
{The quantities $\gf(x_k)$ and $d_k=x_k-x_{k-1}$ used in HB (Section~\ref{sec::mathformalism}) are translation \emph{invariant} since for example $\nabla \hat{f}(\hat{x}) = \gf(x)$, which makes HB translation equivariant (see the proof in Appendix~\ref{app::priors})}.
This shows that it is often possible to make $\C$ translation invariant (\ie $\C(f,x,S) = C(\hat{f},\hat{x},\hat{S})$) which is then enough to make the algorithm translation equivariant, since by direct induction $\hat{y}_k = \M_\theta(\hat{I}_k) = \M_\theta(I_k) =y_k$ and thus $\U(\hat{I}_k,\hat{y}_k,\Gamma) = \U(I_k,y_k,\Gamma)$.

\paragraph{Practical consequences} An easy way to ensure translation invariance of $\C$ is to never output ``absolute'' quantities such as $x_k$ but always differences like $d_{k}=x_k-x_{k-1}$, exactly like HB.
We follow this strategy as it does not put any restriction on $\M_\theta$ nor $\U$.

\subsection{Permutations}

In optimization, the ordering of coordinates is often arbitrary. For example, the functions $(x,y)\mapsto x^2+2y^2$ and $(x,y)\mapsto 2x^2+y^2$ represent the same problem with permuted coordinates. This transformation is represented by a  \emph{permutation matrix} $P$ which contains only zeros except one element equal to $1$ per line, and such that $P^TP=\mathbb{I}_n$ where $\mathbb{I}_n$ denotes the identity matrix of size $n$.
Fix such $P$ and consider the corresponding transformation (with $\hat{f}$ and $\hat{x}$ redefined accordingly). As shown in Table~\ref{tab::equiv}, almost all popular algorithms are permutation equivariant.\footnote{A notable exception regards algorithms constructing block-diagonal matrices like K-FAC \citep{martens2015optimizing}, these blocks depend on the ordering of coordinates.}
\begin{principle}\label{princ::permut}
    Algorithm~\ref{algo::GenL2O} should be equivariant to all permutation matrices $P$.
\end{principle}

\paragraph{Strategy} Remark that this time to get $\hat{x}_k = P x_k$ for all $k\in\N$, we need $\hat{d}_k = Pd_k$ \textit{i.e.} we need equivariance. Taking again the example of HB, we show in Appendix~\ref{app::priors} that $\nabla \hat{f}(\hat{x})=(P^{-1})^T\gf(x) = P\gf(x)$, and we should expect $\hat{d}_k = Pd_k$ (since this is what we want to obtain). Therefore we design $\C$ to be equivariant to permutations, as this is what makes HB permutation equivariant. In L2O, we would get $\M_\theta(\hat{I}_k) = \M_\theta(PI_k)$, so we make $\M_\theta$ equivariant as well ($\M_\theta(PI_k) = P\M_\theta(I_k)$) so that $\U$ gets only permuted quantities, and finally design $\U$ to preserve equivariance.

\paragraph{Practical consequences}
Permutation equivariance strongly restricts the choices for $\M_\theta$ as \citet{zaheer2017deep} showed that the only fully-connected (FC) layer operating along the dimension $n$ and preserving equivariance is very basic with only two learnable scalars and no bias.
To be usable in any dimension $n$, many L2O models rely on per-coordinate predictions \citep{andrychowicz2016learning}, making them permutation equivariant but completely neglecting interactions between coordinates. In Section~\ref{sec::L2OQN} we propose a model that is permutation equivariant while allowing such interactions.

\paragraph{The orthogonal group} Permutations matrices form a subset of the set of orthogonal matrices (square matrices with $P^TP=\mathbb{I}_n$). They correspond to  so-called \emph{rotations} and \emph{reflections}. Several algorithms are equivariant to all orthogonal transformations, which makes this property appealing. Yet, we prove in Appendix~\ref{app:orthoL2O} that is hardly compatible with L2O since it does not hold for any FC layer with ReLU activation function. Fortunately, in many setting the canonical coordinate system has a clear meaning (e.g. each coordinate represents a weight of a neural network to train). So this requirement is \emph{desirable but not crucial}, as the performance of ADAM in deep learning attests.

\begin{figure*}[t]
	\centering
	\begin{minipage}{.7\linewidth}
		\centering
		\includegraphics[width=\linewidth]{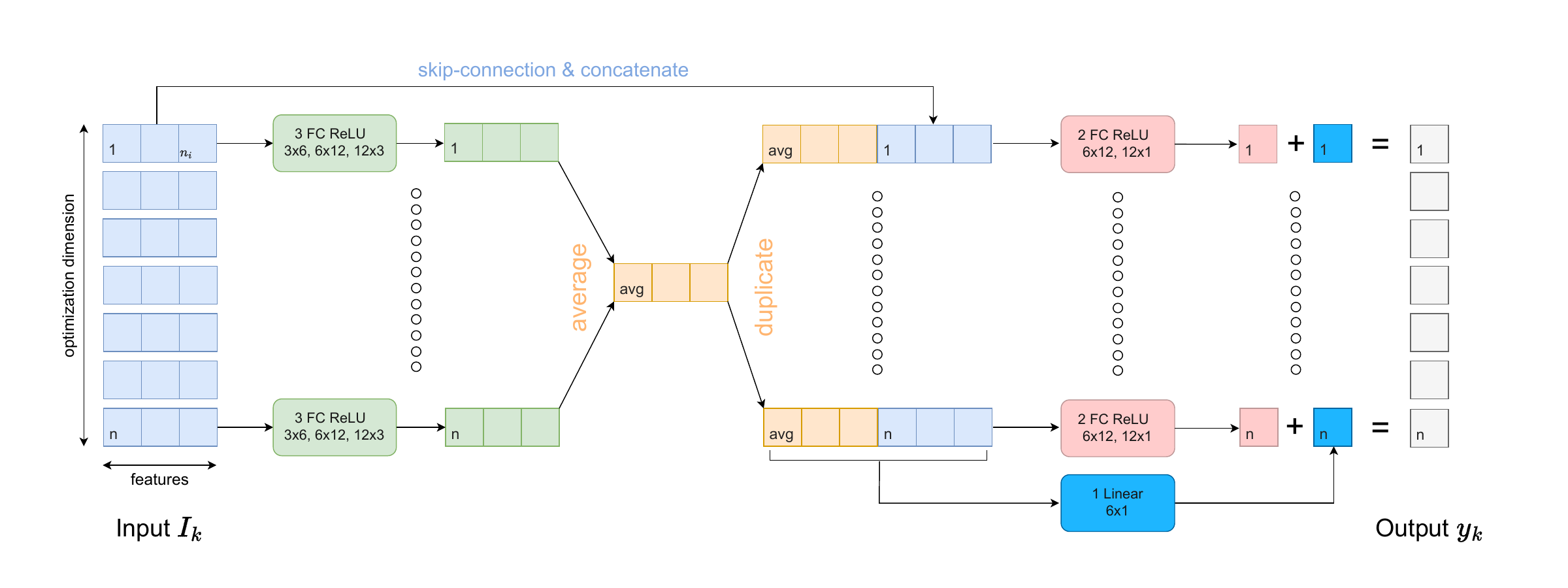}
	\end{minipage}
	{\begin{tabular}{ccc}
			Layer & Size &
			\\
			\midrule
			FC & 3 x 6
			\\
			FC & 6 x 12
			\\
			FC & 12 x 3
			\\
			avg. \&
			duplic. & ---
			\\
			FC & 3 x 12
			\\
			\makecell{FC \&\\
				linear} & \makecell{12 x 1 \&	\\ 6 x 1}
			\\
			\bottomrule
		\end{tabular}
	}
	\caption{Architecture of our L2O model that preserves the principles of Section~\ref{sec::principles}. {The same layers are applied to each coordinate, and interactions between coordinates are allowed via an intermediate averaging and duplication.} \label{fig::model}}
\end{figure*}

\subsection{Rescaling}

Let $\lambda>0$ and consider the geometric rescaling $\T_\lambda(x)=\lambda x$, and redefine $\hat{f}$ and $\hat{x}$ accordingly.
\begin{principle}\label{princ::geomscale} Algorithm~\ref{algo::GenL2O} should be equivariant to geometric rescaling.
\end{principle}
This principle is also optional as it is usually not satisfied by first-order methods. Indeed, one can see that $\hat{f}(\hat{x}_k)=\frac{1}{\lambda}\gf(x_k)$ and we want $\hat{d}_k = \lambda d_k$. So for example one needs to tune $(\alpha,\gamma)$ in HB to recover equivariance (indicated by $\gmark$ in Table~\ref{tab::equiv}). In contrast, Newton's and QN methods are equivariant to geometric rescaling.
In the context of LOA, we construct an algorithm (in Section~\ref{sec::L2OQN}) where we make $\M_\theta$ scale equivariant and show that our update $\U$ can then make Principle~\ref{princ::geomscale} hold. This mildly restricts $\M_\theta$, as \eg $\mathrm{ReLU}$ is scale equivariant but the sigmoid is not.

We consider a final, different, principle. For $\lambda>0$, if \emph{the function} is rescaled: $\hat{f} = \lambda f$, then $\nabla \hat{f} = \lambda \gf$. This does not transform the initialization: $\hat{x}_0=x_0$ so we want \emph{invariance} of the algorithm (and equivariance of function values).
\begin{principle}\label{princ::scale} Algorithm~\ref{algo::GenL2O} should be such that $\A(\lambda f,x_0,\hat{S}_0,K) = \A(f,x_0,S_0,K)$, $\forall \lambda>0$.
\end{principle}
For similar reasons as geometric rescaling, this principle must be dealt with case-specifically, but is compatible with LOA as we show (see Theorem~\ref{thm::principles}).

{
\begin{remark}
ADAM is popular in deep learning (DL) because tuning its step-size is easier than for GD \citep{sivaprasad2020optimizer}, which comes from its invariance to function rescaling (Principle~\ref{princ::scale}). ADAM does not suffer from lacking Principle~\ref{princ::geomscale} in DL, thanks to scaled initialization strategies \citep{he2015delving}.
\end{remark}
}

\subsection{Comparison to Common Heuristics}\label{ref::connections}

{ The principles stated in Section~\ref{sec::principles} {provide a justification for heuristics} that people have used in the L2O literature. For example a large part of the existing work followed \citet{andrychowicz2016learning} and used coordinate-wise models, thus enabling Principle~\ref{princ::main}. Models also rely mostly on $\nabla f$, which has the translation invariance property, crucial for Principle~\ref{princ::translation} (see Section~\ref{sec::translation}). Notably, \citet{wichrowska2017learned} discusses the scaling issue (Principle~\ref{princ::scale}) and tries to mitigate it by decomposing inputs in magnitude and unit directions. Similarly, \citet{lv2017learning} used what they call ``random rescaling'' (Section 4.1 therein), which is a data-augmentation technique that can exactly be interpreted as an attempt to learn Principle~\ref{princ::geomscale}. Our work brings justification for these heuristics and provides design strategies to replace them.
}


\section{APPLICATION TO LEARNING QUASI-NEWTON ALGORITHMS}\label{sec::L2OQN}

{We illustrate our approach on the example of building a LOA, based on the BFGS method and prove it obeys the principles above.} BFGS is a quasi-Newton (QN) method, \ie one that iteratively builds an approximation of the computationally-expensive inverse Hessian matrix used in Newton's method. This is thus in line with the idea to adapt to each problem \emph{on the fly} (see Section~\ref{sec::principles}). While combining L2O and QN methods has been considered before  (see Section~\ref{ref::relatedwork}), our approach differs in several aspects starting with the following.

\subsection{A variational view on BFGS\label{sec::variational}}
Let $k\in\N$ be an iteration, we use the notation $\dg{k} = \gf(x_k)-\gf(x_{k-1})$ and recall that $d_k = x_k-x_{k-1}$.
 QN methods are based on the fact that for quadratic functions the \emph{secant equation} $d_k = \Hf(x_k)^{-1}\dg{k}$ holds. QN methods aim to iteratively build approximations $B_k$ to $\Hf(x_k)^{-1}$, with the constraint that $B_k$ must preserve the secant equation: $d_k = B_k\dg{k}$ and that $B_k$ is symmetric. From a variational perspective, \cite{greenstadt1970variations,goldfarb1970family} showed that BFGS aims to keep $B_k$ close to the previous approximation $B_{k-1}$ by taking $B_k$ as the solution of the following problem:
\begin{equation}\label{eq::BFGSprob}
	B_k = 
		\min_{\substack{ B\in\R^{n\times n}
				\\s.t.
				 B_{k}\dg{k} = d_{k}\text{ and }
 B_k -B_k^T = 0}} \norm{B-B_{k-1}}_W.
\end{equation}
Here $\norm{\cdot}_W$ denotes the Frobenius norm reweighted by some symmetric positive definite matrix $W$.
Denoting $y_k = W^{-1}\dg{k}$ and $r_k = d_k - B_{k-1}\dg{k}$, one can show that the solution of \eqref{eq::BFGSprob} is
\begin{multline}\label{eq::genBFGSupdate}
	B_k = B_{k-1} \\+ \frac{1}{\langle\dg{k}, y_k\rangle} \left[ r_{k} y_k^T + y_k r_{k}^T - \frac{\langle\dg{k}, r_{k}\rangle}{\langle\dg{k}, y_k\rangle} y_k y_k^T \right].
\end{multline}
BFGS is then based on the heuristic trick  (albeit elegant) that taking $W^{-1}$ to be the \emph{unknown} next approximation $B_k$, yields $y_k=d_k$ (due to the secant equation) and preserves positive-definiteness.
{
Instead of using L2O to directly predict the matrix $B_k$ as done in prior work, we use it precisely at this stage.
We use a model $\M_\theta$ that \emph{predicts directly} a different $y_k$ than that of BFGS. Remark that there is no need to predict the matrix $W$ since is only appears through the vector $y_k = W^{-1}\dg{k}$. This allows enhancing BFGS with L2O while preserving the coherence of the algorithm.
}

\subsection{Our Learned Algorithm \label{sec::ourL2O}}
We now specify each part of our LOA-BFGS method.

\paragraph{The oracle $\C$}
For each iteration $k\in\N$ of our algorithm, we use the state $S_k = \{x_{k-1}, \gf(x_{k-1}), B_{k-1}\}$. Our oracle $\C$ computes $\nabla f(x_k)$, $d_k$ and $\dg{k}$ as in BFGS but also new features $B_{k-1}\Delta g_k$ and $-\gamma B_{k-1}\nabla f(x_k)$, gathered as $I_k = \left(B_{k-1}\Delta g_k, d_k, -\gamma B_{k-1}\nabla f(x_k)\right)$. Note that all these features must be scale invariant since $B_{k-1}$ approximates $\Hf(x_k)^{-1}$.
\paragraph{The learned model $\M_\theta$}
Our model only takes three features as input ($\nf=3$) but creates additional ones by applying a block of coordinate-wise FC layers, without bias, then averaging the result and concatenating it with $I_k$. This allows feature augmentation and makes each coordinate interacting with all the others. The result is then fed to another coordinate-wise block of FC layers (no bias) added to a linear layer yielding the output $y_k\in \R^n$. The architecture is detailed in Appendix~\ref{app:model} and summarized in Figure~\ref{fig::model}.
The linear layer acts as a skip-connection and will allow us to introduce a trick that stabilizes the training later in Section~\ref{sec::practical}.
Note that the cost of each operation inside $\M_\theta$ is proportional to $n$ which is cheaper than the matrix-vector products of cost $O(n^2)$ that $\C$ and $\U$ (and vanilla BFGS) involve.

\paragraph{The update $\U$ and storage $\S$} Our update is that of BFGS:\footnote{To fit in our mathematical formalism, BFGS and our algorithm would need $\U$ to also take $S_k$ as input. Since this would not affect any of the discussion above, we ignored this for the sake of simplicity.} the  approximation $B_k$ is updated using \eqref{eq::genBFGSupdate} with a different $y_k$, and $x_{k+1} = x_k - \gamma B_k\gf(x_k)$, where $\gamma>0$ is a step-size, usually close to $1$ (or chosen by line-search). Like in vanilla BFGS, $\S$ finally stores $\{x_{k}, \gf(x_{k}), B_{k}\}$ for the next iterate.

\paragraph{Initialization of $B_{-1}$} QN methods require an initial approximated inverse Hessian matrix $B_{-1}$. While the simplest choice is  $\mathbb{I}_n$, several works \citep{NIPS2012_e034fb6b,BFO19} observed notable improvement by initializing with the Barzilai-Borwein (BB) step-size \citep{barzilai1988two}:
$
  \gamma_{\mathrm{BB}}^{(0)} \stackrel{\mathrm{def}}{=} \frac{\langle \dg{0},d_0\rangle}{\norm{\dg{0}}^2}.
$
We follow this approach and take $B_{-1} = 0.8\gamma_{\mathrm{BB}}^{(0)}\mathbb{I}_n$. With this choice, $B_{-1}$ agrees with Principles~\ref{princ::geomscale} and~\ref{princ::scale} \emph{without additional knowledge on $f$} (see Appendix~\ref{app:proofPrinc}) and so do Algorithm~\ref{algo::L2OQN} and BFGS.

The whole process is summarized in Algorithm~\ref{algo::L2OQN}. Note that replacing the model $\M_\theta$ with $y_k=d_k$ makes Algorithm~\ref{algo::L2OQN} coincide exactly with BFGS. This will prove useful for training $\M_\theta$ (see Section~\ref{sec::practical}).

\begin{algorithm}
	\DontPrintSemicolon
	\SetKwFor{For}{for}{:}{} 
	\SetKwProg{Prog}{Compute}{:}{}
	\SetKwInput{Input}{input}
	\SetKwInput{Given}{given}
	\SetKwInput{Initialize}{initialize}
	\Given{model $\M_\theta$ defined in Figure~\ref{fig::model} and Sections~\ref{sec::ourL2O} and~\ref{app:model}}
	\Input{function to minimize $f$, initialization $x_0$}
	\Input{initial state $S_{0} = \{x_{-1}, \gf(x_{-1}), B_{-1}\}$, \ (with $B_{-1}=0.8\gamma^{(0)}_{\mathrm{BB}}\mathbb{I}_n$)}
	\Input{number of iterations $K$, step-size $\gamma$ (default value $\gamma=1$).}
	\setstretch{1.2}
	\For{$ k = 0 $ to $K-1$}{
		\Prog{$\C(f,x_k,S_k)$}{
			$\dg{k}=\gf(x_k)-\gf(x_{k-1})$\;
			$d_{k} =x_k-x_{k-1}$\;
			$ I_k \gets \left(B_{k-1}\Delta g_k, d_k, -\gamma B_{k-1}\nabla f(x_k)\right)$ \;
			\Return{$I_k$}\;
		}
		\BlankLine
		\Prog{$\M_\theta(I_k)$}{
			\Return{$y_k$}\;
		}
		\BlankLine
		\Prog{Update step $\U$}{
			$r_{k} = d_{k} - B_{k-1}\dg{k}$\;
			$B_k = B_{k-1} + \frac{1}{\langle\dg{k}, y_k\rangle} \left[ r_{k} y_k^T + y_k r_{k}^T - \frac{\langle\dg{k}, r_{k}\rangle}{\langle\dg{k}, y_k\rangle} y_k y_k^T \right]$ \;
			\BlankLine
			$ x_{k+1} \gets  x_k -\gamma B_k \gf(x_k) $\;
		}
		\BlankLine
		\Prog{Storage $\S$}{
			$S_{k+1} = \left\{
			x_{k}, \gf(x_{k}), B_{k}
			\right\}$ \;
		}
		\BlankLine
	}
	\Return{$x_K$}
	\caption{Learning enhanced QN Algorithm\label{algo::L2OQN}}
\end{algorithm}

\begin{figure*}[t]
	\centering
	\setlength\tabcolsep{2pt}
		\setlength\tabcolsep{1pt}
		\renewcommand{\arraystretch}{0.4}
		\begin{tabular}{ccccccc}
			\raisebox{1.3cm}{\rotatebox{90}{$ \frac{f(x_k)-f^\star}{f(x_0)-f^\star}$}}
			&
			\includegraphics[height=.22\linewidth]{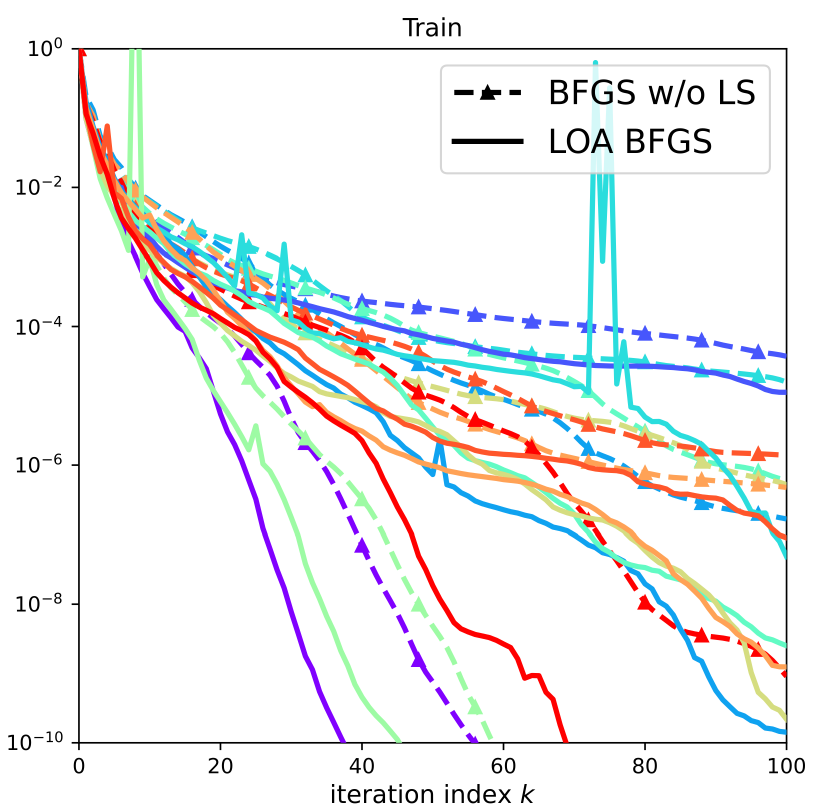}
			& \includegraphics[height=.22\linewidth]{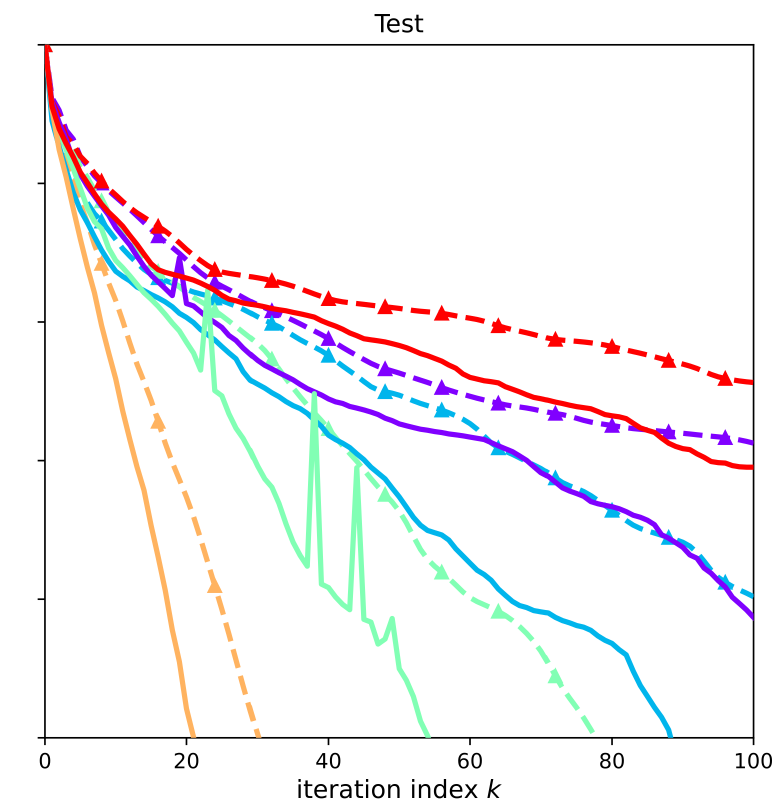}
			&
			\includegraphics[height=.22\linewidth]{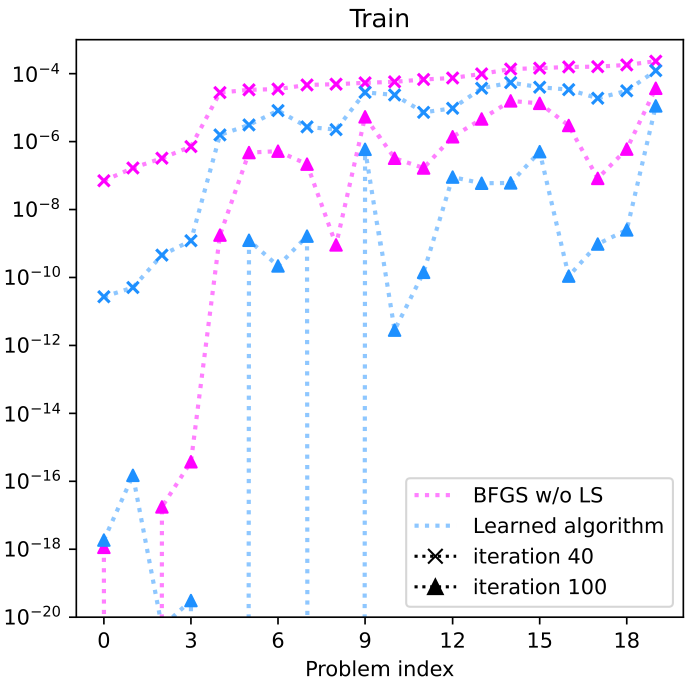}
			&
			\includegraphics[height=.22\linewidth]{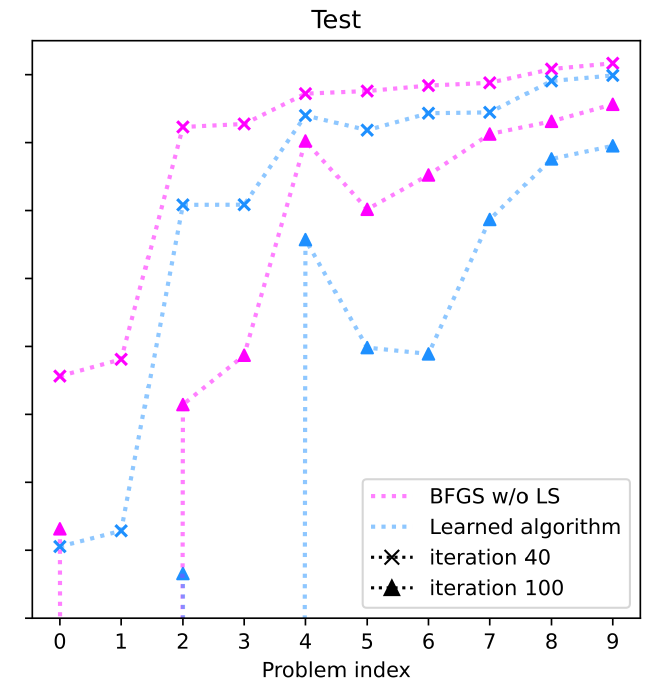}
		\end{tabular}
		\caption{Performance of our learned BFGS method on quadratic functions in dimension $n=100$ (training and test). Left plots show relative sub-optimality gap against iterations, each color represents a different problem. Right plots: relative sub-optimality gap for each problem at several stages (lower is better). \label{fig::quad100}
		}
	\end{figure*}

\subsection{Theoretical Analysis\label{sec::theoretical}}

Based on our strategy from Section~\ref{sec::principles}, our LOA follows the principles therein, as proved in Appendix~\ref{app:proofPrinc}.
\begin{theorem}\label{thm::principles}
  With the choice of $B_{-1}$ above, Algorithm~\ref{algo::L2OQN} follows Principles~\ref{princ::main}-\ref{princ::translation}-\ref{princ::permut}-\ref{princ::geomscale}-\ref{princ::scale}.
\end{theorem}

Another benefit of preserving and enhancing existing algorithms is that their coherent structures allow deriving convergence results, proved later in Appendix~\ref{app:proofs}.
\begin{theorem}\label{thm::conv}
Assume that $f$ has $L$-Lipschitz continuous gradient and that for all $k\in\N$, $B_k$ is positive definite with eigenvalues lower and- upper-bounded by $c,\ C>0$ respectively. Then for any step-size $\gamma \leq  \frac{2}{CL}$, $(f(x_k))_{k\in\N}$ converges and $\lim_{k\to+\infty} \norm{\gf(x_k)} = 0$.
\end{theorem}
It is important to note that Theorem~\ref{thm::conv} is more restrictive than usual convergence theorems. It is indeed based on strong assumptions regarding the eigenvalues of the matrices $(B_k)_{k\in\N}$.
Yet, since $B_k$ is constructed based on the output of the model $\M_\theta$, the failure can only come from the learning part of the algorithm. One could thus optionally enforce the assumption by design of $\M_\theta$ in the fashion of \citep{moeller2019controlling,heaton2023safeguarded}. This would put additional restrictions on the model and does not seem necessary in the experiments below.

\begin{figure*}[t]
	\centering
	\raisebox{-0.5cm}{\rotatebox{90}{$ \frac{f(x_k)-f^\star}{f(x_0)-f^\star}$}}
	\begin{minipage}{0.25\linewidth}
		\centering
		\includegraphics[width=\linewidth]{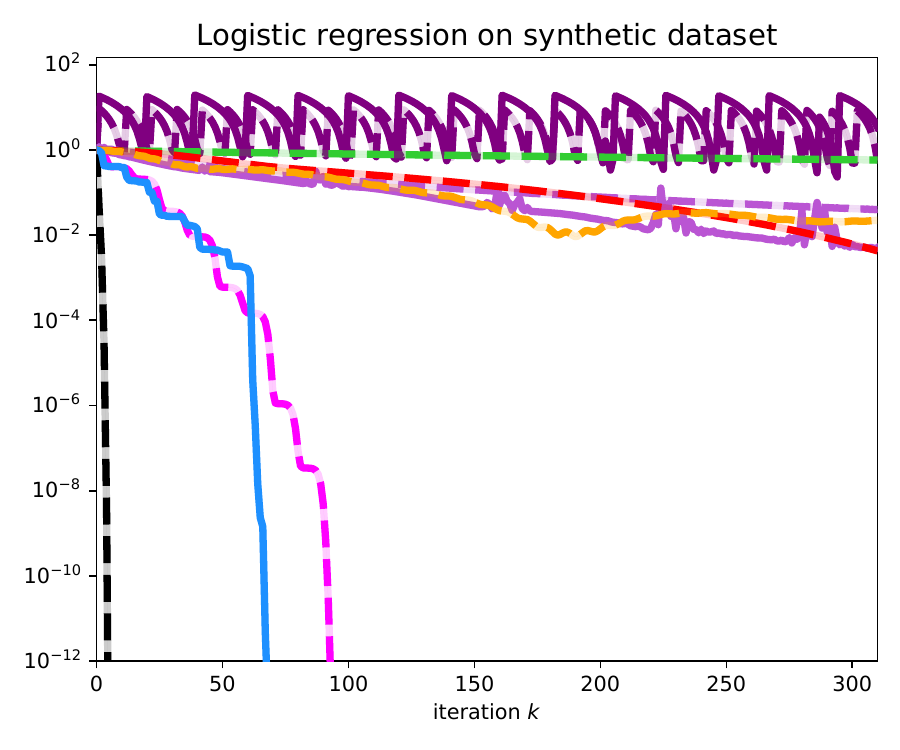}
	\end{minipage}
	\begin{minipage}{0.25\linewidth}
		\centering
		\includegraphics[width=\linewidth]{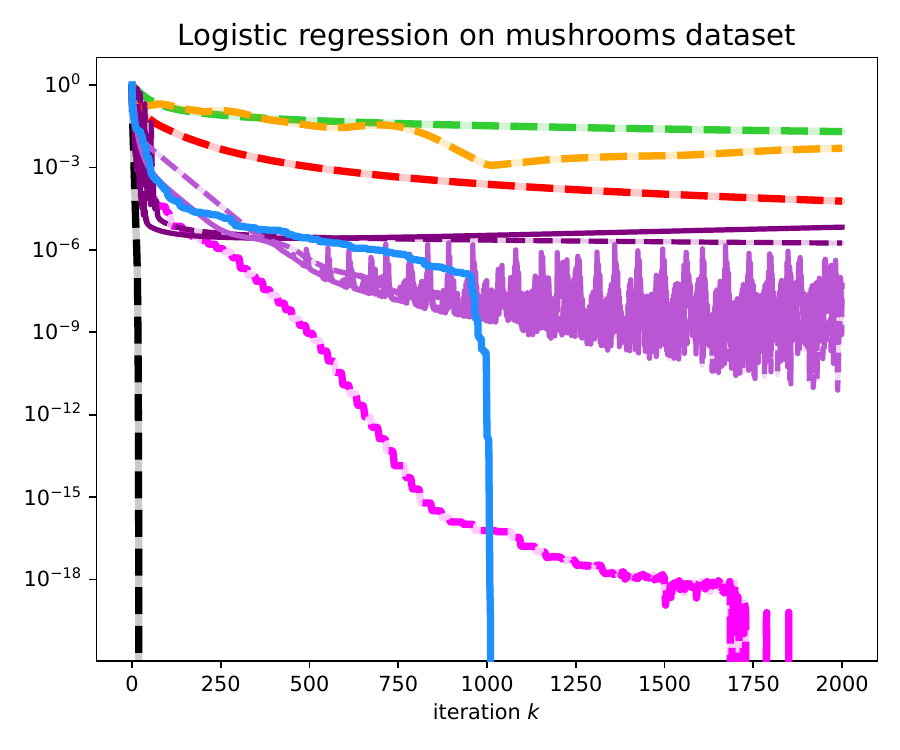}
	\end{minipage}
	\begin{minipage}{0.25\linewidth}
		\centering
		\includegraphics[width=\linewidth]{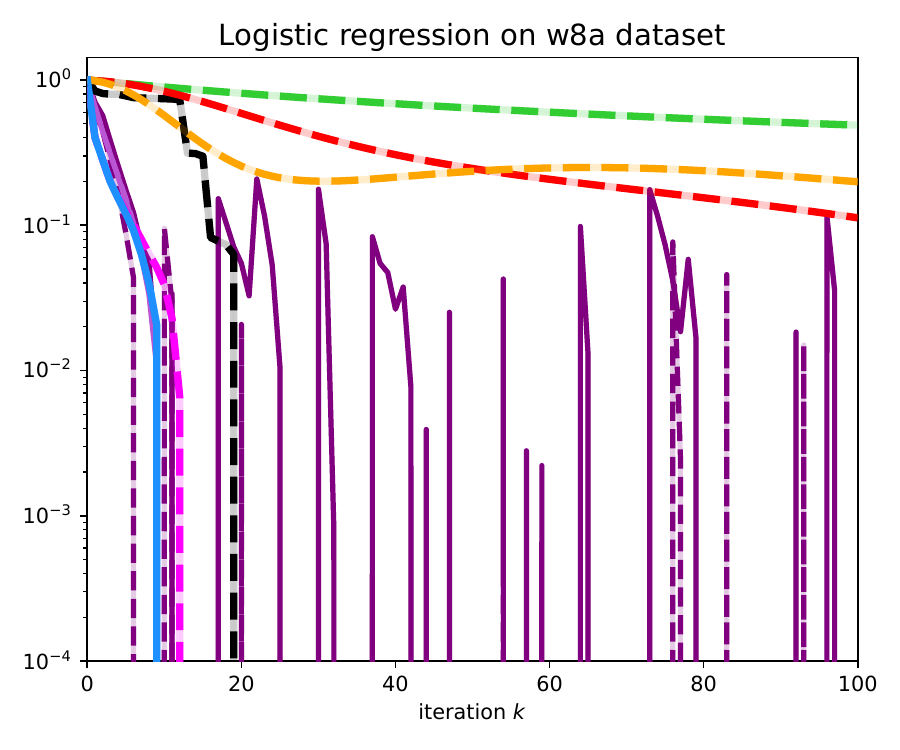}
	\end{minipage}
	\\
	\raisebox{-0.5cm}{\rotatebox{90}{$ \frac{f(x_k)-f^\star}{f(x_0)-f^\star}$}}
	\begin{minipage}{0.25\linewidth}
		\centering
		\includegraphics[width=\linewidth]{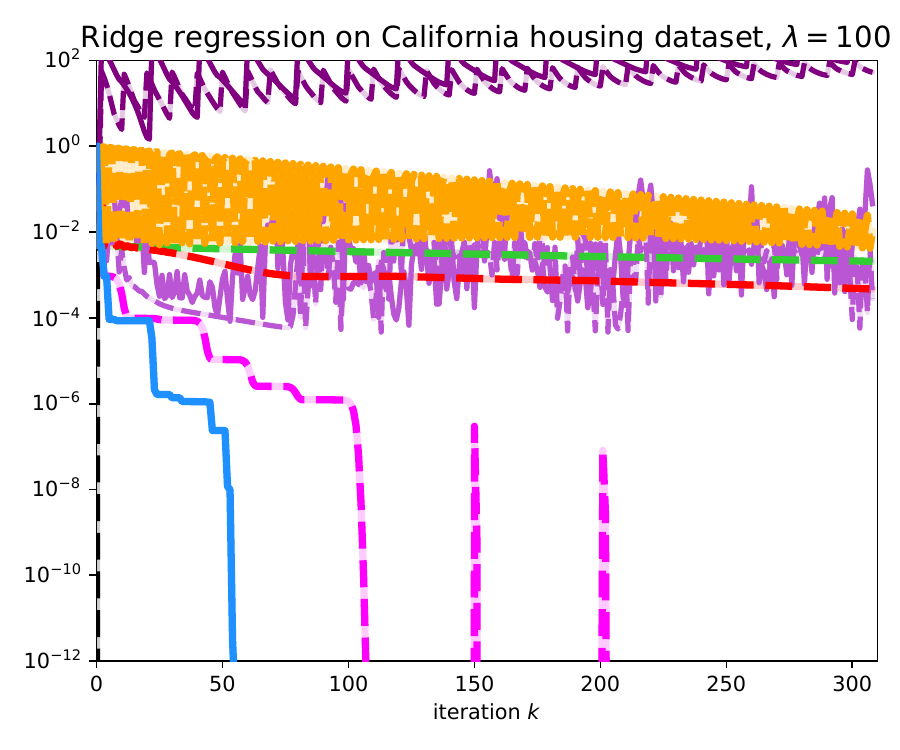}
	\end{minipage}
	\begin{minipage}{0.01\linewidth}
		\ 
	\end{minipage}
	\begin{minipage}{0.22\linewidth}
		\begin{flushright}
			\includegraphics[width=.95\linewidth]{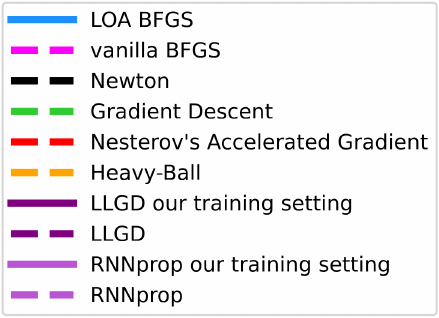}
		\end{flushright}
	\end{minipage}
	\begin{minipage}{0.005\linewidth}
		\ 
	\end{minipage}
	\begin{minipage}{0.25\linewidth}
		\centering
		\includegraphics[width=\linewidth]{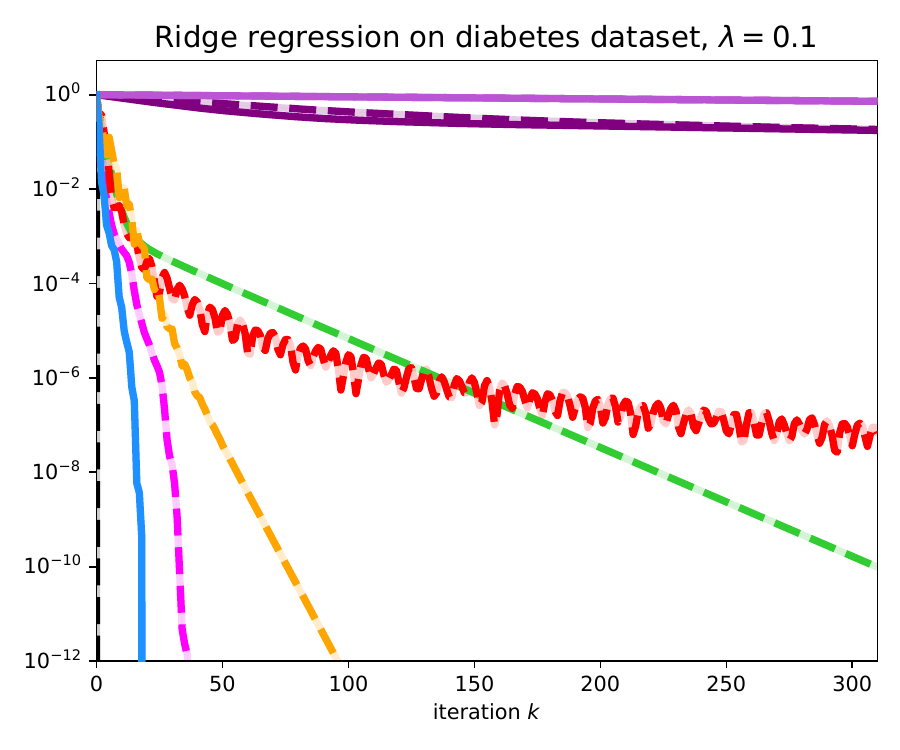}
	\end{minipage}
	\caption{{Comparison of our method against BFGS, other hand-crafted algorithms, and L2O methods (see Appendix~\ref{app:L2Obaselines}), on different types of problems (detailed in Appendix~\ref{app:detailsOnExp}) \emph{without retraining} $\M_\theta$. Our LOA always outperforms vanilla BFGS, evidencing its ability to work far beyond the training setting.} \label{fig::realworld}}
\end{figure*}

\section{EXPERIMENTS AND PRACTICAL CONSIDERATIONS}\label{sec::practical}
In addition to the design choices that we already made to follow our principles (special model, use of ReLU, no bias, etc.), we discuss practical considerations that ease the training of our model $\M_\theta$. In what follows we consider a training set of $D\in\N$ problems indexed by superscripts $(f_j,x_0^j,S_0^j)$, for $j=1,\ldots,D$. For each problem we run the algorithm for $K\in\N$ iterations and denote by $(x_k^j)_{k\in \{0,\ldots,K\}}$ the resulting sequence.

\paragraph{Loss function}
Our loss function is based on the last values $(f_j(x_K^j))_{j\in\{1,\ldots,D\}}$. However, we make it invariant to the optimal value by using the sub-optimality gap $f_j(x_K^j)-f_j^{\star}$ where $f_j^{\star}$ is the minimum of $f_j$. A key element to take into account in optimization is that the magnitude of function values may heavily vary between problems and across iterations. This can be slightly mitigated by normalizing by the initial sub-optimality gap $f_j(x_0^j)-f_j^{\star}$, however we instead propose to run vanilla BFGS for $K$ iterations as well and normalize by its sub-optimality gap. After averaging over all problems our loss function is:
\begin{equation}\label{eq::lossfunction}
  \mathcal{L}(\theta) = \frac{1}{D}\sum_{j=1}^D \log\left(1 + \frac{f_j(x_K^j(\theta))-f_j^{\star}}{f_j(\tilde{x}_K^j)-f_j^{\star}}\right),
\end{equation}
where $\tilde{x}^j_K$ is the $K$-th iteration of BFGS ran on the same problem. 
This is related to the idea of \citet{chen2020training} who trained by competing against a baseline, the novelty in \eqref{eq::lossfunction} is the use of relative function values.

We make the loss even more robust to different magnitudes by applying a $\log(1+\cdot)$ composition. We emphasize that the algorithm does not make use of $f_j^{\star}$, which is only used at training time.

\paragraph{Initialization of $\M_\theta$} Training L2O models is notoriously difficult as the loss function may quickly explode \citep{wang2021guarantees}. Our approach allows for a specific trick that dramatically stabilizes training. Indeed, according to  Section~\ref{sec::variational}, BFGS is a special case of our algorithm in which $y_k=d_k$. By initializing the weights of the last FC layer to zero and the linear layer to be $(0,0,0,0,1,0)$ one can check that our algorithm is initialized to exactly coincide with BFGS before training. According to \eqref{eq::lossfunction}, the initial value of the loss function is always $\log(2)$ which dramatically stabilizes the training as shown in Figure~\ref{fig::train} in Appendix~\ref{app::Exp}. To the best of our knowledge this is a new strategy.

\paragraph{Methodology and results} We construct a training set of $D=20$ problems made of ill-conditioned quadratics functions in dimension $n=100$ with eigenvalues generated at random and random initializations. The details are provided in Apppendix~\ref{app:detailsOnExp} and the execution time of each algorithm is reported in Appendix~\ref{app::wallclock}. We train our model for $K=40$ iterations. We then evaluate the performance of the algorithm on several settings that differ from the training one: more than $40$ iterations, in-distribution test problems, quadratics in larger dimension, and finally on five other problems including logistic and ridge losses, and some real-world datasets (Figure~\ref{fig::realworld}). 
We compare to several hand-crafted algorithms, as well as L2O baselines \citep{andrychowicz2016learning,lv2017learning}, see Appendix~\ref{app:L2Obaselines}.
The code for reproducing the results, including the trained weights of our LOA are available in a public repository.\footnote{\url{https://github.com/camcastera/L2OtoLOA}}

Looking first at the training setting in Figure~\ref{fig::quad100}, observe that our L2O model improves upon BFGS for every problem after $40$ iterations, sometimes by several orders of magnitudes. This significant improvement transfers to $100$ (trained only for $40$) iterations for almost all problems, and also to the test set, as well as on quadratic functions in dimension $500$ (see Figure~\ref{fig::quad500} in Appendix~\ref{app::Exp}).

{Figure~\ref{fig::realworld}, shows that our LOA does not break on different objective functions and datasets, but even improves upon BFGS despite not having been trained on these. Algorithm~\ref{algo::L2OQN} is also more robust than the other L2O algorithms considered: they exhibit slow decrease or heavy oscillations on most of the problems of Figure~\ref{fig::realworld}. The experiments evidence the benefits of the strategy proposed in Section~\ref{sec::principles} to achieve generalization}.

\section{CONCLUSION}
We provided a new approach for designing more robust learned optimization algorithms. Our work blends all aspects of L2O: from optimization theory to machine learning models, including implementation and training considerations. We illustrate how promising the approach is in practice by applying our techniques to build a L2O-enhanced BFGS algorithm. It results in an algorithm outperforming vanilla BFGS consistently beyond the training setting. Enhancing existing algorithms allowed us to provide preliminary theoretical guarantees, which most L2O algorithms lack of, as well as a new training strategy that significantly eases the training and mitigates the difficulty of training L2O models.
{Our approach is generic and can be applied to almost any algorithm (\eg HB discussed in Section~\ref{sec::mathformalism}). This work thus calls for exploring many directions, such as enhancing other algorithms, designing more advanced models, adding new principles. An important next step is to adapt the principles and the recipe to the case of stochastic algorithms.
}
\FloatBarrier

\acknowledgments

This work is supported by the ANR-DFG joint project TRINOM-DS under number DFG-OC150/5-1. Part of this work was done while C.\ Castera was with the department of mathematics of the University of Tübingen, Germany.
We thank the anonymous reviewers for their suggestions that helped improving the experiments.
C. Castera thanks Michael Sucker for useful discussions and Armin Beck for pointing a mistake in a proof.
We also thank the development teams of the following libraries: Python \citep{rossum1995python}, Matplotlib \citep{hunter2007matplotlib} and Pytorch\citep{paszke2019pytorch}.

\clearpage

\bibliographystyle{plainnat}
\bibliography{biblio.bib}

\newpage
\appendix

\onecolumn

\textbf{\Large Supplementary Material}

\setcounter{section}{7} 
\setcounter{equation}{7} 

\section{DETAILED ANALYSIS OF EQUIVARIANCE OF POPULAR ALGORITHMS}\label{app::priors}

Below we detail how to obtain the properties listed in Table~\ref{tab::equiv}. We begin by studying the transformations.

\subsection{The Chain-rule}

In this section we detail how each transformation affects the derivatives of $f$. All these results are based on the chain-rule. For an invertible mapping $\T\colon\R^n\to\R^n$, we  study the function $\hat{f} = f\circ\T^{-1}$. The chain rule states that $\forall y\in\R^n$
$$
D_y (\hf) = D_y (f\circ \T^{-1}) = D_{\T^{-1}(y)}(f) \cdot D_y (\T^{-1}),
$$
where $D_y(\hf)$ is the Jacobian of $\hf$ at $y$. Rewriting this in terms of gradients (the transpose of the Jacobian):
$$
\nabla\hf(y) =  (D_y (\T^{-1}))^T \gf(\T^{-1}(y)).
$$
In most of what follows we will apply the chain rule above to the point $\hat{x}=\T(x)$, which yields
\begin{equation}\label{eq::}
	\nabla\hf(\hat{x}) =  (D_{\T(x)} (\T^{-1}))^T \gf(\T^{-1}(\T(x))) = (D_{\T(x)} (\T^{-1}))^T \gf(x),
\end{equation}
so the Jacobian of $\T^{-1}$ captures how the gradient is transformed. We now detail this for each transformation.

\subsection{List of Transformations\label{app:listoftransfo}}
We consider the transformations in Table~\ref{tab::equiv}. For each case we redefine $\T$ and, without restating it, define $\hat{f} = f\circ \T^{-1}$ and $\hat{x}=\T(x)$, for all $x\in\R^n$.

\textbf{Translation.} Let $v\in\R^n$ and for all $x\in\R^n$, consider the translation $\T(x) = x+v$.
Then one can see that $D (\T^{-1}) = \mathbb{I}_n$ which implies that $(D_{\T(x)} (\T^{-1}))^T = \mathbb{I}_n$. So, $\nabla\hat{f}(\hat{x}) = \gf(x)$ and similarly, one can show that $\nabla^2 \hf(\hx) = \Hf(x)$.

\textbf{Orthogonal Linear Transformations.} Let $P\in \R^{n\times n}$ an orthogonal matrix ($P^TP=\mathbb{I}_n$) and $\T\colon x\in\R^n\mapsto Px$. Then using the orthogonality of $P$, $\T^{-1}(x) = P^{-1}x =P^T x$. It is a linear mapping, so $D \T^{-1} = P^T$ and $(D_{\T(x)} (\T^{-1}))^T) = (P^T)^T = P$. Therefore $\nabla\hf(\hx) = P\gf(x)$. Similarly, using the linearity of $\T^{-1}$, we can show that $\nabla^2\hf(\hx) = P\Hf(x)P^T$.

\textbf{Permutations.} Permutation matrices are a specific type of orthogonal matrices, therefore the above directly applies.

\textbf{Geometric Rescaling}. Let $\lambda>0$ and consider the transformation $\T\colon x\in\R^n \mapsto \lambda x$. Then $\T^{-1}(x) = \frac{1}{\lambda}x$ which is again a linear mapping so $D \T^{-1} = \frac{1}{\lambda} \mathbb{I}_n$. We deduce as before that $\nabla\hf(\hx) = \frac{1}{\lambda}\gf(x)$ and $\nabla^2\hf(\hx) = \frac{1}{\lambda^2}\Hf(x)$.

\textbf{Function Rescaling}. Let $\lambda>0$, when considering $\hat{f} = \lambda f$, the linearity of the differentiation directly gives $\nabla\hf = \lambda \gf$ and $\nabla^2\hf = \lambda \Hf$.

\subsection{Analysis of Popular Algorithms\label{app:PrincOtherAlgos}}
We now show the properties of Table~\ref{tab::equiv} for each algorithm therein, except for BFGS which is analyzed together with Algorithm~\ref{algo::L2OQN} later in Section~\ref{app:proofPrinc}. Each time, the proofs are done by induction. We can safely assume that $x_0$ and $S_0$ are properly adapted so that the induction holds for $k=0$ (this was explained in Section~\ref{sec::principles}).

To show each equivariance property (or invariance in the case of function rescaling), we fix $k\in\N$ and assume that the equivariance holds for all iterates up to $k$, and then prove that it still holds at iteration $k+1$.

\paragraph{Gradient descent and HB}
We already extensively discussed the properties of HB which was our running example in Section~\ref{sec::principles}. As for gradient descent, it is simply HB with $\alpha=0$. Using the results of Section~\ref{app:listoftransfo} one can straightforwardly deduce translation, permutations and orthogonal equivariances. Thus we only discuss the case of rescaling.
For $\lambda>0$ and $\hat{f}=f(\frac{1}{\lambda}\cdot)$, assuming that the induction hypothesis $\hx_k=\lambda x_k$ holds, we previously showed that, $\nabla\hat{f}(\hx_k)=\frac{1}{\lambda}\gf(x_k)$, so the iteration of HB reads,
$$
\hat{x}_{k+1}  = \hx_k + \alpha \hat{d}_k + \gamma \nabla\hf(\hx_k) = \lambda x_k +\lambda \alpha  d_k + \frac{\gamma}{\lambda}\gf(x_k).
$$
So for $\lambda\neq 1$ we see that $\hat{x}_{k+1}\neq \lambda x_{k+1}$. This can be fixed however by tuning $\gamma$ specifically for each problem (we would get $\hat{\gamma} = \lambda^2\gamma$). The case of function rescaling $\hat{f}=\lambda f$ is almost identical.

\paragraph{Newton's method} The update of Newton's method reads
\begin{equation*}
	x_{k+1} = x_k - \left[\nabla^2f(x_k)\right]^{-1}\gf(x_k).
\end{equation*}
As above, translation equivariance is straightforward. As for orthogonal matrices $P$, using the results from Section~\ref{app:listoftransfo} it holds that
\begin{align*}
	\hx_{k+1}
	&= \hx_k - \left[\nabla^2\hf(\hx_k)\right]^{-1}\nabla f(\hx_k)
	= P x_k - \left[P\nabla^2 f(x_k)P^T\right]^{-1} P\gf(x_k)
	\\
	&= P x_k - (P^T)^{-1}\left[\nabla^2 f(x_k)\right]^{-1} P^{-1} P\gf(x_k) = P x_k - P\left[\nabla^2 f(x_k)\right]^{-1}\gf(x_k),
\end{align*}
which proves the equivariance.

For geometric rescaling by $\lambda>0$, remark from Section~\ref{app:listoftransfo} that the inverse Hessian is rescaled by $\lambda^2$ and the gradient is rescaled by $1/\lambda$, so the result follows. The same is true for invariance with respect to function rescaling.

\paragraph{The ADAM Algorithm}
The iterations of the ADAM algorithm read
\begin{equation*}
	\begin{cases}
		m_k &= \beta_1 m_{k-1} + (1-\beta_1)\gf(x_k)
		\\ v_k^2 & = \beta_2 v_{k-1}^2 + (1-\beta_2)\gf(x_k)\odot \gf(x_k)
		\\
		x_{k+1} & = x_k - \gamma\frac{m_k}{\sqrt{v_k^2}}
	\end{cases},
\end{equation*}
where $\beta_1,\beta_2\in [0,1)$, $\gamma>0$ is the step-size, $\odot$ denotes the element-wise product, the square root and quotient are applied element wise, and $m_{-1},v_{-1}^2\in\R^n$.

Again, translation equivariance is straightforward using the results from section~\ref{app:listoftransfo}. The robustness with respect to the function rescaling is also easy to check in that case since both $m_k$ and $\sqrt{v_k^2}$ are rescaled like $\gf$, hence no need to adapt $\gamma$ unlike HB and GD. However, for the same reason we see that equivariance w.r.t.\ geometric rescaling does not hold, except if we adapt $\gamma$.
We now consider an orthogonal matrix $P$. According to section~\ref{app:listoftransfo} and assuming that up to iteration $k$ ADAM is equivariant to orthogonal transformations, we have $\hf(\hx) = P\gf(x)$ and $\hat{m}_k = Pm_k$. However, looking at $v_k^2$, note that
\begin{equation}\label{eq::ADAM}
	\nabla\hf(\hx_k) \odot \nabla\hf(\hx_k) =  (P\gf(x_k))\odot(P\gf(x_k))
\end{equation}
which in general is not equal to $P (\gf(x_k)\odot\gf(x_k))$. Therefore, for most orthogonal matrices we do not have $\hat{v}_k = P v_k$ and equivariance does not hold.

Nevertheless, in the special case where $P$ is a permutation matrix, each line of $P$ contains exactly one coefficient equal to one and all others are zero. Since $\odot$ is an element-wise operation, one can check that we then have $ (P\gf(x_k))\odot(P\gf(x_k)) = P(\gf(x_k)\odot\gf(x_k))$ such that $\hat{v}_k = Pv_k$. Applying the same reasoning to all other element-wise operations in \eqref{eq::ADAM}, we deduce that, despite not being equivariant to all orthogonal matrices, ADAM is permutation equivariant.

\subsection{On the Difficulty of Preserving Orthogonal Equivariance for LOA\label{app:orthoL2O}}

The discussion above regarding ADAM shows why preserving equivariance to all orthogonal matrices is very hard for LOA since even element-wise operations may not commute with orthogonal matrices.

To give an example, let $P$ be an orthogonal matrix, let $y\in\R^n$ and $\sigma\colon\R\to\R$ be a \emph{non-linear} activation function (to be applied element wise in layers of a neural network). To preserve equivariance with respect to $P$, one would want that $\sigma(Py) = P\sigma(y)$, which, for the $i$-th coordinate reads,
\begin{equation}\label{eq::LOAortho}
	\sigma\left(\sum_{j=1}^n P_{i,j} y_j\right) = \sum_{j=1}^n P_{i,j}\sigma(y)_j.
\end{equation}
Yet, since $\sigma$ is assumed to be non-linear, it does not commute with the sum (which would still not even be sufficient for \eqref{eq::LOAortho} to hold).
Therefore, if $y$ is the output of an FC layer (used coordinate-wise like in Algorithm~\ref{algo::L2OQN}) in a neural network, then applying a non-linear activation function (\eg $\mathrm{ReLU}$ or sigmoid), we see that equivariance with respect to $P$ is broken. This shows that orthogonal equivariance is not even compatible with coordinate-wise FC layers and hence hardly possible to achieve for LOA.

Finally, note that when $P$ is a permutation matrix, then $\forall i\in\{1,\ldots,n\}$, there exists $l\in\{1,\ldots,n\}$ such that $P_{i,l}=1$ and for all $j\neq l$, $P_{i,j}=0$. So \eqref{eq::LOAortho} becomes
\begin{equation}\label{eq::LOApermut}
	\sigma\left(P_{i,l} y_l\right) = P_{i,l}\sigma(y)_l   \iff   \sigma\left( y_l\right) = \sigma(y)_l,
\end{equation}
and since $\sigma$ is applied element wise, \eqref{eq::LOApermut} holds true. So permutation equivariance is more compatible with LOA than orthogonal equivariance.

\section{PROOF OF THEOREM~\ref{thm::principles}}\label{app:proofPrinc}
\begin{proof}[Proof of Theorem~\ref{thm::principles}]
	Principle~\ref{princ::main} holds by construction of the algorithm where $n$ is not used to choose $p$ nor $\nf$.
	We now prove that Principles~\ref{princ::translation} to~\ref{princ::scale} hold one by one. As for other algorithms in Section~\ref{app:PrincOtherAlgos}, in each case we explicitly state which transformation $\T$ is considered and implicitly redefine $\hat{f}=f\circ\T^{-1}$ and define $(\hat{x}_k)_{k\in\N}$ as the iterates of the algorithm applied to $(\hat{f},\hat{x}_0,\hat{S}_0)$ (all quantities with a ``hat'' symbol are defined accordingly). We again proceed by induction: we fix $k\in\N$ and assume that equivariance (or invariance for Principle~\ref{princ::scale}) holds up to iteration $k$ and show that it still holds at iteration $k+1$. We also show that the principles hold at $k=0$ by construction.

	Unlike the  algorithms discussed in Section~\ref{app:PrincOtherAlgos}, Algorithm~\ref{algo::L2OQN} and BFGS additionally use a matrix $B_{k}$, (stored in the state $S_{k+1}$). Since $B_k$ aims to approximate the inverse Hessian $\Hf(x_k)^{-1}$, we expect $B_k$ to be transformed by $\T$ in the same way as $\Hf(x_k)^{-1}$ is (see Section~\ref{app:listoftransfo}). We will prove that this is the case, again by induction.


	\textbf{Translation.}
	Let $v\in\R^n$ and the translation $\T\colon x\in\R^n\mapsto x+v$.
	Assume that $\forall i\leq k$, $\hx_i = \T(x_i) = x_i+v$ and that $\forall i\leq k-1$, $\hat{B}_i = B_i$. Then $\hat{d}_k = \hat{x}_k-\hat{x}_{k-1} = d_k$ and we showed in Section~\ref{app:listoftransfo} that $\nabla\hf(\hx_k) = \gf(x_k)$. Similarly, $\widehat{\Delta g}_k = \dg{k}$.
	So,
	$$
	\hat{I}_k = \left(\hat{B}_{k-1}\widehat{\Delta g}_k, \hat{d}_k, -\gamma \hat{B}_{k-1}\nabla \hf(\hx_k)\right) =  \left(B_{k-1}\Delta g_k, d_k, -\gamma B_{k-1}\nabla f(x_k)\right) = I_k,
	$$
	This is not surprising as we explained in Section~\ref{sec::principles} that we constructed $\C$ so that the above is true. Then we directly deduce $\hat{y}_k = \M_\theta(\hat{I}_k) = \M_\theta(I_k) = y_k$ and the rest of the proof follows.

	As for the case $k=0$, by construction (see Section~\ref{sec::principles}), $\hx_0=x_0+v$ and $\hx_{-1} = x_{-1}+v$. One can then easily check that our choice of $B_{-1}$ (defined in Section~\ref{sec::ourL2O}) is translation invariant. So by induction, Principle~\ref{princ::translation} holds.


	\textbf{Permutation.}

	Let $P$ a permutation matrix of $\R^n$ and let $\T\colon x\in\R^n\mapsto Px$.
	Assume that $\forall i\leq k$, $\hx_i = \T(x_i) = Px_i$ and that $\forall i\leq k-1$, $\hat{B}_i = P B_iP^T$.
	Note that the hypothesis on $B_i$ matches that of the inverse Hessian in Section~\ref{app:listoftransfo}.
	Then $\hat{d}_k = Px_k-Px_{k-1} = Pd_k$. We showed in Section~\ref{app:listoftransfo} that $\nabla\hf(\hx_k) = P\gf(x_k)$, and similarly, $\widehat{\Delta g}_k = P\dg{k}$.
	So,
	$$
	\hat{I}_k = \left(PB_{k-1}P^TP\Delta g_k, Pd_k, -\gamma PB_{k-1}P^TP\nabla f(x_k)\right) = PI_k,
	$$
	\ie $\C$ is permutation equivariant as intended. Then $\hat{y}_k = \M_\theta(\hat{I}_k) = \M_\theta(PI_k)$ and as justified in Appendix~\ref{app:orthoL2O}, all the operations applied element-wise in $\M_\theta$ are permutation equivariant, and the averaging also is. So $\M_\theta$ is permutation equivariant, \ie $\hat{y}_k = Py_k$.

	Regarding the step $\U$, we recall the notation $r_k = d_k - B_{k-1}\dg{k}$ used in \eqref{eq::genBFGSupdate}. Remark that $\hat{r}_k = Pd_k - PB_{k-1}\dg{k} = Pr_k$, and substituting $\hat{y}_k$, $\widehat{\Delta g}_k$ and $\hat{r}_k$ in \eqref{eq::genBFGSupdate} (and using again $P^TP=\mathbb{I}_n$), we obtain $\hat{B}_k = PB_kP^T$ and $\hx_{k+1} = Px_{k+1}$.

	Finally, at $k=0$, by construction $\hx_0=Px_0$ and $\hx_{-1}=Px_{-1}$ and one can easily check that $\hat{\gamma}_{\mathrm{BB}}^{(0)} = \gamma_{\mathrm{BB}}^{(0)}$ (again due to $P$ being orthogonal), such that $\hat{B}_{-1} = B_{-1}$. So Principle~\ref{princ::permut} holds true.


	\textbf{Geometric rescaling.}
	Let $\lambda>0$ and let $\T\colon x\in\R^n\mapsto \lambda x$.
	Assume that $\forall i\leq k$, $\hx_i = \T(x_i) = \lambda x_i$ and that $\forall i\leq k-1$, $\hat{B}_i = \lambda ^2 B_i$ (as in Section~\ref{app:listoftransfo}).
	Then $\hat{d}_k = \lambda x_k- \lambda x_{k-1} = \lambda  d_k$. We also showed in Section~\ref{app:listoftransfo} that $\nabla\hf(\hx_k) = \frac{1}{\lambda}\gf(x_k)$, thus $\widehat{\Delta g}_k = \frac{1}{\lambda}\dg{k}$.
	So,
	$$
	\hat{I}_k = \left(\lambda^2 B_{k-1}\frac{1}{\lambda}\Delta g_k, \lambda d_k, -\gamma \lambda^2 B_{k-1}\frac{1}{\lambda}\nabla f(x_k)\right) = \lambda I_k,
	$$
	which means that $\C$ is equivariant as we prescribed. Then our model $\M_\theta$ is a composition of linear operations and $\mathrm{ReLU}$ activation functions which are all equivariant to rescaling, so the model is equivariant, \ie $\hat{y}_k = \lambda y_k$. Plugging this into the update step we obtain
	\begin{equation*}
		\hat{B}_k = \lambda^2 B_{k-1} + \frac{1}{\langle\lambda^{-1}\dg{k}, \lambda y_k\rangle} \left[ \lambda^2 r_{k} y_k^T + \lambda^2 y_k r_{k}^T - \frac{\langle\lambda^{-1}\dg{k}, \lambda r_{k}\rangle}{\langle\lambda^{-1}\dg{k},\lambda y_k\rangle} \lambda^2 y_k y_k^T \right] = \lambda^2 B_k.
	\end{equation*}

	Finally, the case $k=0$ holds by construction and thanks to the BB step-size since
	$$\hat{\gamma}_{\mathrm{BB}}^{(0)} = \frac{\langle \lambda^{-1}\dg{0},\lambda d_0\rangle}{\lambda^{-2}\norm{\dg{0}}^2} = \lambda^2\gamma_{\mathrm{BB}}^{(0)}.$$
	This shows how the choice of $B_{-1}$ is crucial to preserve equivariance to rescaling. Overall Principle~\ref{princ::geomscale} holds.


	\textbf{Function rescaling.}
	Let $\lambda>0$ and consider $\hat{f}=\lambda f$. For this last principle we want to prove invariance of the algorithm. Therefore assume that $\forall i\leq k$, $\hx_i = x_i$ and that $\forall i\leq k-1$, $\hat{B}_i = \frac{1}{\lambda} B_i$ (it scales like the inverse Hessian).
	Then $\hat{d}_k = d_k$ and we also have $\nabla\hf(\hx_k) = \lambda \gf(x_k)$, thus $\widehat{\Delta g}_k = \lambda \dg{k}$. So
	$$
	\hat{I}_k = \left(\frac{1}{\lambda} B_{k-1}\lambda \Delta g_k, d_k, -\gamma \frac{1}{\lambda} B_{k-1}\lambda\nabla f(x_k)\right) = I_k,
	$$
	so $\C$ is invariant, which directly implies $\hat{y}_k = y_k$ and then
	\begin{equation*}
		\hat{B}_k = \frac{1}{\lambda} B_{k-1} +  \frac{1}{\langle\lambda\dg{k}, y_k\rangle} \left[ r_{k} y_k^T + y_k r_{k}^T - \frac{\langle\lambda\dg{k},  r_{k}\rangle}{\langle\lambda\dg{k}, y_k\rangle}  y_k y_k^T\right] = \frac{1}{\lambda} B_k.
	\end{equation*}

	Finally, the case $k=0$ holds again thanks to the use of the BB step-size to initialize $B_{-1}$, which proves that Principle~\ref{princ::scale} holds and concludes the proof.
\end{proof}

\begin{remark}
	The proof above can easily be applied to BFGS since it corresponds to the special case where $\M_\theta$ is replaced by $y_k=d_k$.
\end{remark}

\section{PROOF OF THEOREM~\ref{thm::conv}}\label{app:proofs}
\begin{proof}[Proof of Theorem~\ref{thm::conv}]
	Assume that $f$ has $L$-Lipschitz continuous gradient, that is, for all $x,y\in\R^n$,
	$$
	\norm{\gf(x)-\gf(y)}\leq L \norm{x-y}.
	$$
	Then the \emph{descent lemma} (see \eg \cite{garrigos2023handbook}) states that for all $x,y\in\R^n$,
	\begin{equation}\label{eq::descentlemma}
		f(y) \leq f(x) + \langle \gf(x),y-x\rangle + \frac{L}{2}\norm{y-x}^2.
	\end{equation}
	Now let $(x_k)_{k\in\N}$ and $(B_k)_{k\in\N}$ be respectively the sequence of iterates and the matrices generated by Algorithm~\ref{algo::L2OQN} applied to $(f,x_0,S_0)$. Using the descent lemma \eqref{eq::descentlemma}, we get,
	\begin{align*}
		&f(x_{k+1}) \leq f(x_k) + \langle \gf(x_k),-\gamma B_k\gf(x_k)\rangle + \frac{L}{2}\gamma^2\norm{B_k\gf(x_k)}^2,
	\end{align*}
	which we rewrite
	\begin{equation}\label{eq::QNdescentlemma}
		f(x_{k+1}) \leq f(x_k) + \left\langle B_k\gf(x_k) ,-\gamma\gf(x_k) + \frac{L}{2}\gamma^2 B_k\gf(x_k)\right\rangle.
	\end{equation}
	By construction, (see \eqref{eq::genBFGSupdate}), $B_k$ is real symmetric, so there exists an orthogonal matrix $P_k\in\R^{n\times n}$ and a diagonal matrix $D_k\in\R^{n\times n}$ such that,
	$$
	B_k = P_kD_kP_k^T.
	$$
	Using this in \eqref{eq::QNdescentlemma}, we obtain
	\begin{equation*}\label{eq::QNdescentDiag}
		f(x_{k+1}) \leq f(x_k) + \left\langle P_kD_kP_k^T\gf(x_k) ,-\gamma P_kP_k^T\gf(x_k) + \frac{L}{2}\gamma^2 P_kD_kP_k^T\gf(x_k)\right\rangle,
	\end{equation*}
	where we used the fact that $P_kP_k^T=\mathbb{I}_n$ to write $\gf(x_k) = P_kP_k^T\gf(x_k)$.
	We denote $g_k = P_k^T\gf(x_k)$ and get:
	\begin{align}
		\nonumber
		f(x_{k+1}) \leq f(x_k) + \left\langle P_k D_k g_k ,-\gamma P_k g_k + \frac{L}{2}\gamma^2 P_k D_k g_k\right\rangle
		\\ \label{eq::QNdescentDiag2}
		\iff f(x_{k+1}) \leq f(x_k) + \left\langle D_k g_k ,-\gamma g_k + \frac{L}{2}\gamma^2 D_k g_k\right\rangle,
	\end{align}
	where we used the fact that $P_k$ is orthogonal in the last line. Since $D_k$ is orthogonal, denoting by $(g_{k,i})_{i\in\{1,\ldots,n\}}$ and $(b_{k,i})_{i\in\{1,\ldots,n\}}$ the coordinates of $g_k$ and the eigenvalues of $B_k$, respectively, we deduce that
	\begin{align*}
		\left\langle D_k g_k ,-\gamma g_k + \frac{L}{2}\gamma^2 D_k g_k\right\rangle
		&=  \sum_{i=1}^n b_{k,i} g_{k,i}^2 \left(-\gamma + \frac{L}{2}\gamma^2 b_{k,i}\right)
		=\gamma \sum_{i=1}^n b_{k,i} g_{k,i}^2 \left(\frac{L}{2}\gamma b_{k,i}-1\right)
		\\
		&\leq \gamma \sum_{i=1}^n b_{k,i} g_{k,i}^2 \underbrace{\left(\frac{L}{2}\gamma C-1\right)}_{\leq 0}
		\leq 0,
	\end{align*}
	where for the last line we used the assumption that for all $k\in\N$ and $\forall i\in\{1,\ldots,n\}$, $0<c\leq b_{k,i}\leq C$ and that $\gamma\leq \frac{2}{CL}$.
	We use this in \eqref{eq::QNdescentDiag2}:
	\begin{equation}\label{eq::finaldescent}
		f(x_{k+1}) \leq f(x_k) - \gamma \left(1-\frac{L}{2}\gamma C\right)\sum_{i=1}^n b_{k,i} g_{k,i}^2 \leq f(x_k).
	\end{equation}
	So the sequence $(f(x_k))_{k\in\N}$ is non-increasing, and since $f$ is  a lower-bounded function, then $(f(x_k))_{k\in\N}$ converges.

	We now sum \eqref{eq::finaldescent} from $k=0$ to $K\in\N$,
	\begin{align}
		\nonumber
		\sum_{k=0}^{K} f(x_{k+1}) - f(x_k) \leq -\gamma\left(1-\frac{L}{2}\gamma C\right)\sum_{k=0}^{K}  \sum_{i=1}^n b_{k,i} g_{k,i}^2
		\\
		\iff \gamma\left(1-\frac{L}{2}\gamma C\right)\sum_{k=0}^{K}  \sum_{i=1}^n b_{k,i} g_{k,i}^2 \leq f(x_0) - f(x_{K+1}).\label{eq::sumdescent}
	\end{align}
	Since $f$ is lower bounded, the right-hand side of \eqref{eq::sumdescent} is uniformly bounded, so
	\begin{equation}\label{eq::boundedSum}
		\gamma\left(1-\frac{L}{2}\gamma C\right)
		\sum_{k=0}^{+\infty}  \sum_{i=1}^n b_{k,i} g_{k,i}^2 <+\infty
		\iff \sum_{k=0}^{+\infty} \langle B_k\gf(x_k),\gf(x_k)\rangle <+\infty.
	\end{equation}
	Finally, by assumption $B_k$ is positive definite with eigenvalues uniformly lower-bounded by $c>0$, therefore \eqref{eq::boundedSum} implies that $\sum_{k=0}^{+\infty} c\norm{\gf(x_k)}^2 <+\infty$, and thus in particular $\lim_{k\to+\infty} \norm{\gf(x_k)}=0$.

\end{proof}



\section{ADDITIONAL DETAILS ON THE EXPERIMENTS}\label{app:detailsOnExp}
In this section we provide additional details on how to reproduce the experiments of Section~\ref{sec::practical}.

\subsection{The Model of Algorithm~\ref{algo::L2OQN}\label{app:model}}

The neural network used is exactly that described in Figure~\ref{fig::model}, we simply detail the FC and linear blocks.
The first coordinate-wise FC block is made of $3$ layers with output shapes $(6,12,3)$, the second-one has $2$ layers with output shapes $(12,1)$. We use $\mathrm{ReLU}$ activation functions for each layer except for the last layer of each block.
The linear layer is of size $6\times 1$, again with no bias.
The total number of parameter of the network is $216$. In comparison, the training set is made of $20$ problems in dimension $n=100$, thus $p=216$  is much smaller than $20\times 100 = 2000$. We also apply the algorithm to problems in dimension $500$ where even for a single problem $p<500$.

\subsection{Problems and Datasets\label{app::dataset}}

\paragraph{Quadratic functions}
To generate a quadratic function in dimension $n$, we proceed as follows. We create a matrix $A$ by first sampling its largest and smallest eigenvalues $\lambda_\mathrm{min}$, $\lambda_\mathrm{max}$ uniformly at random in $[0.1,1]$ and  $[1,50]$ respectively. We then generate the $n-2$ other eigenvalues of $A$ uniformly at random in $[\lambda_\mathrm{min},\lambda_\mathrm{max}]$. This gives us a diagonal matrix $D£$ containing the eigenvalues of $A$.
We then generate another matrix $B$ with Gaussian $\mathcal{N}(0,1)$ entries and make it symmetric via $B\gets B+B^{T}$. We then compute the orthogonal matrix $P$ that diagonalizes $B$ and use $P$ to build $A=PDP^T$. We then sample a vector $b\in\R^n$ whose entry are sampled uniformly at random in $[0,15]$. Our function $f$ finally reads: $f\colon x\in\R^n \mapsto \frac{1}{2}\norm{Ax-b}^2$. 
The quadratic functions in dimension $500$ (Figure~\ref{fig::quad500}) are generated with the same process.

With this process, the largest eigenvalue of $\nabla^2 f$ is $\frac{\lambda_{\mathrm{max}}^2}{2}$. In our experiments the largest eigenvalue in any problem is approximately $1159$ and the largest condition number (the ratio between the largest and smallest eigenvalues) is approximately $15156$, hence our dataset includes ill-conditioned problems.

\paragraph{Regularized Logistic Regression}
We consider a binary logistic regression problem, as presented in \cite{hastie2009elements}.
For the left plot of Figure~\ref{fig::realworld}, we generate two clouds of $M$ points sampled from Gaussian distributions $\mathcal{N}(\mu_1,1)$ and $\mathcal{N}(\mu_2,1)$ where $\mu_1$, $\mu_2$ are themselves sampled from $\mathcal{N}(-1,1)$ and $\mathcal{N}(1,1)$ respectively. We store the coordinates of the $2M$ points in a matrix $A\in\R^{2M\times (n+1)}$ (a row of ones is concatenated with $A$, see \citep{hastie2009elements}). We also create a vector $b\in\R^{2M}$ where each $b_i$ takes either the value $0$ or $1$ depending on which class the $i$-th data point belong to. Given these $A$ and $b$, for all $x\in\R^{n+1}$, the function $f$ is defined as:
$$
f(x) = \frac{1}{2M}\sum_{i=1}^{2M}\log\left(1 + e^{x^TA_i} \right) - b_i x^TA_i + \frac{\eta}{2}\norm{x}^2.
$$
The last term is a regularization that makes the problem strongly convex. We use a very small $\eta=10^{-3}$. Our experiments are done for $M=100$ and $n=50$.

The experiments on the \textsf{w8a} and \textsf{mushrooms} datasets (also on Figure~\ref{fig::realworld}) are also logistic regression experiments but with real-world data. The two datasets are publicly available but also provided in our public repository.

\paragraph{Ridge regression} Ridge regression consists in minimizing
$$
	f(x) = \norm{Ax-b}^2 + \frac{\lambda}{2}\norm{x}^2,
$$
where $\lambda>0$ is a regularization parameter, $A\in\R^{m\times n}$ and $b\in\R^m$. This is also a type of quadratic function, however the two experiments on the right of Figure~\ref{fig::realworld} are based on the \textsf{California housing} and \textsf{diabetes} datasets, which are real-world applications, hence different from the synthetic data we trained on. Moreover, in these datasets, the matrix $A$ is usually not square and $m>n$, which makes the solution to $Ax=b$ not unique. This also differs from the training setting.

\subsection{The BFGS Baseline}
For a fair comparison, the BFGS algorithm is implemented exactly like Algorithm~\ref{algo::L2OQN} but with $y_k=d_k$ instead of using learned model. We use the same strategy for initializing $B_{-1}$.
For both algorithms we generate a random starting point $x_{-1}\in\R^n$, and perform a gradient descent step along $\gf(x_{-1})$ to obtain the true initialization $x_0\in\R^n$. Both algorithms thus always start at the same $x_0$ with the same state $S_0 = \{x_{-1},\gf(x_{-1}),B_{-1}\}$. When using our algorithm with fixed step-size (during training and in most experiments), we compare it to BFGS with fixed step-size. When BFGS is used with line-search, then so is our algorithm.

\subsection{The L2O Baselines\label{app:L2Obaselines}}
We consider two popular L2O baselines: the method from \cite{andrychowicz2016learning}, referred to as LLGD and the RNNprop method from \cite{lv2017learning}. The implementation provided in our public repository is adapted from \citet{liu2023towards} \url{https://github.com/xhchrn/MS4L2O}. We consider two version of each L2O method: one trained in the training setting from the aforementioned papers, and one retrained on a setting closer to ours (deterministic, quadratic functions in dimension $n=100$). Our repository contains the weights used for each model.

\subsection{Training Strategy}
\textbf{Training set.} Our training dataset is made of $10$ quadratic functions in dimension $n=100$ created following the strategy described in Section~\ref{app::dataset}. We generate two different initializations at random for each function, yielding a training dataset of $20$ problems.

\textbf{Initialization of the network.} We initialize the parameter $\theta$ (the weights of our layers) by following the new initialization strategy introduced in Section~\ref{sec::practical}. Recall that with this strategy, before training $\theta$ our model coincides with BFGS, stabilizing the training process as shown on Figure~\ref{fig::train}.

\textbf{Training by unrolling.} We run the algorithm for $K=40$ iterations and use the loss function $\mathcal{L}(\theta)$ described in \eqref{eq::lossfunction}.
However, we observe in practice that unrolling the last iterate (\ie computing the gradient of $\mathcal{L}$ with respect to the last iterate $K$) is numerically unstable (known as the vanishing/exploding gradient problem). We mitigate this issue by computing $\mathcal{L}(\theta)$ every $5$ iterations (\ie at iterations $\{5,10,\ldots,40\}$) and by ``detaching'' the matrix $B_k$ every $5$ iterations (\ie neglecting the effect that old predictions of the model have on $B_k$). We then average the $8$ values of the loss computed along the trajectory and ``back-propagate'' to compute the gradient of this loss function. Since we neglect the influence that old iterates have on $B_k$, we do not compute ``true gradients'' of $\mathcal{L}$. Yet, it is important to note that this is acceptable since training is only a mean to obtain a good parameter $\theta$. This does not break any of our principles.

\textbf{Training parameters.} We train the model with the ADAM~\citep{kingma2014adam} optimizer with gradient clipping. We do not compute the full gradient $\nabla\mathcal{L}(\theta)$ but a mini-batch approximation of it by selecting only two problems at random at every iteration. We use the learning rate $10^{-4}$ for the FC layers and $10^{-3}$ for the linear layer. We save the model that achieved the best training loss on average over one epoch (a full pass on the training dataset).

\subsection{Computational Architecture\label{app:computeArchi}}
We ran all the experiments on a HP EliteBook 840 with 32 GiB of RAM, and an Intel Core Ultra 5 125U CPU with 12 cores at 4 GHz. No GPUs were used for the experiments.
We used Python~\citep{rossum1995python} 3.12.3, Numpy~\citep{walt2011numpy} 2.1.2 and Pytorch~\citep{paszke2019pytorch} 2.5.0 running on Ubuntu 24.04.

\subsection{Wall-clock Time Estimation \label{app::wallclock}}
We estimated the average wall-clock time per iteration for each algorithm on each problem. We note that the compute times reported are only estimations that are architecture and implementation dependant and may vary. The results reported in Table~\ref{tab::wallclock} were obtained with the architecture described in Section~\ref{app:computeArchi}.
In the table we report the average cost of one iteration for each algorithm relative to that of gradient descent, i.e., $\mathrm{exectime}(\mathrm{algo}) / \mathrm{exectime}(\mathrm{GD})$.

\begin{table}[!ht]
	\centering
	\caption{Average wall-clock time relative to gradient descent\label{tab::wallclock}}
	\begin{tabular}{l|ccccc|c}
		\toprule
		~ & \makecell{ridge\\ california} & \makecell{ridge\\ diabetes} & \makecell{logistic\\ w8a} & \makecell{logistic\\ mushrooms} & \makecell{logistic\\ synthetic} & \makecell{average \\speed-down}\\ 
		\midrule
		Gradient Descent & 1.00 & 1.00 & 1.00 & 1.00 & 1.00 & 1.00 \\ 
		Nesterov & 1.03 & 1.05 & 1.00 & 1.50 & 0.87 & 1.09 \\ 
		Heavy-ball & 1.10 & 1.16 & 1.00 & 2.38 & 0.96 & 1.32 \\ 
		BFGS & 2.40 & 2.99 & 0.97 & 1.41 & 1.70 & 1.89 \\ 
		L2LbyGD & 6.23 & 2.13 & 1.22 & 1.67 & 1.70 & 2.59 \\ 
		LOA BFGS & 3.51 & 3.27 & 1.46 & 2.35 & 2.95 & 2.71 \\ 
		RNNprop & 6.96 & 2.43 & 1.08 & 1.28 & 1.87 & 2.72 \\ 
		Newton & 7.91 & 10.56 & 190.23 & 40.34 & 8.32 & 51.47 \\ 
		\bottomrule
	\end{tabular}
\end{table}

\begin{figure}[ht]
\centering
	\begin{tabular}{cc}
		\raisebox{2cm}{\rotatebox{90}{$ \frac{f(x_k)-f^\star}{f(x_0)-f^\star}$}}
		\includegraphics[height=.3\linewidth]{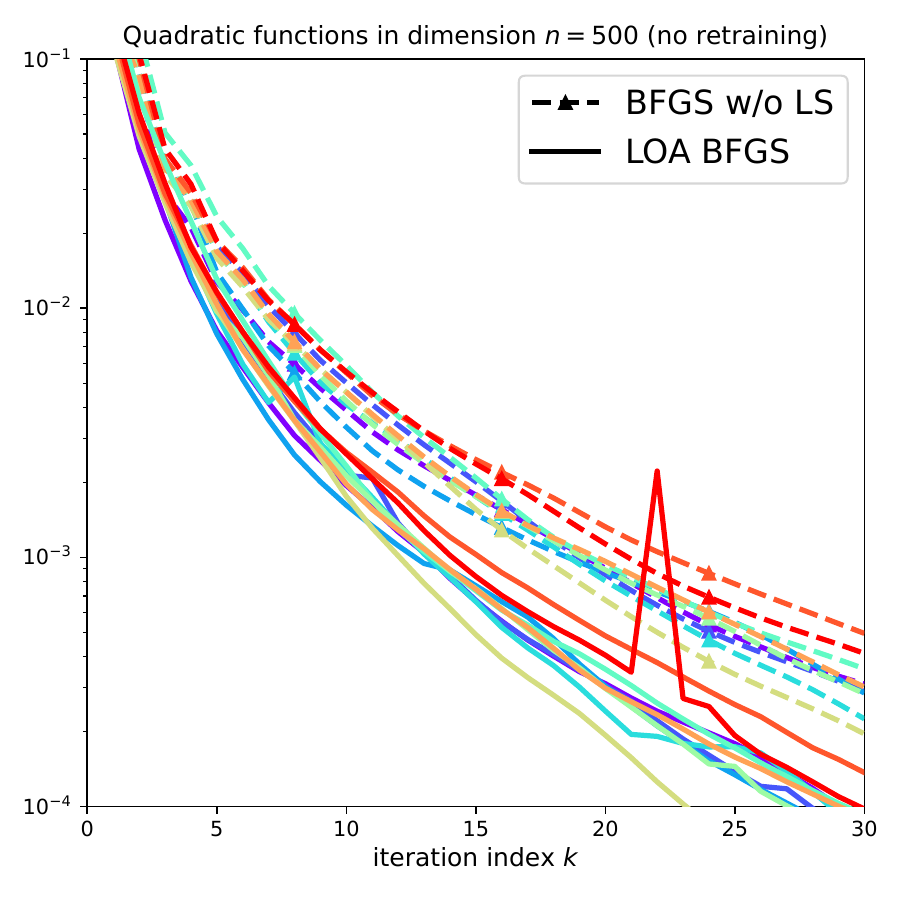}
		&
		\includegraphics[height=.3\linewidth]{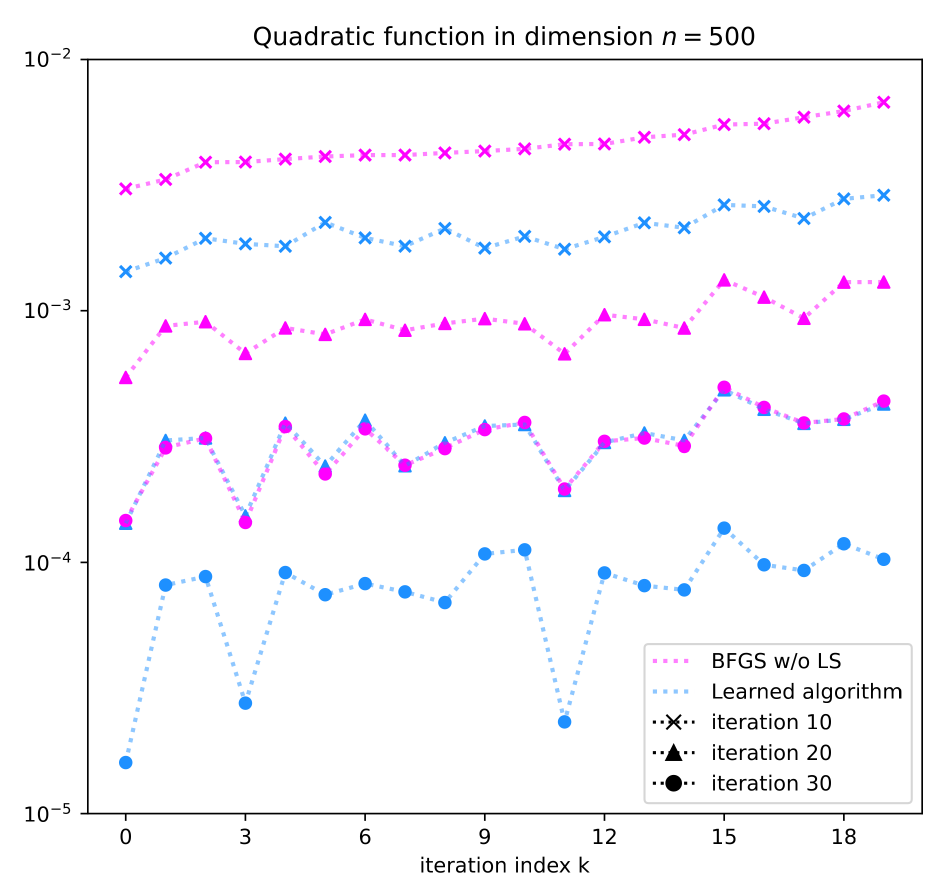}
	\end{tabular}
	\caption{Quadratic problems similar to that of the training setting but used in dimension $n=500$, \emph{without retraining}. Left: sub-optimility gap against iterations. Right: Relative sub-optimality gap for each problem at several stages of the optimization process. \label{fig::quad500}}
\end{figure}

\begin{figure}[ht]
	\centering
	\begin{tabular}{cc}
		\raisebox{2cm}{\rotatebox{90}{$ \frac{f(x_k)-f^\star}{f(x_0)-f^\star}$}}
		\includegraphics[height=.3\linewidth]{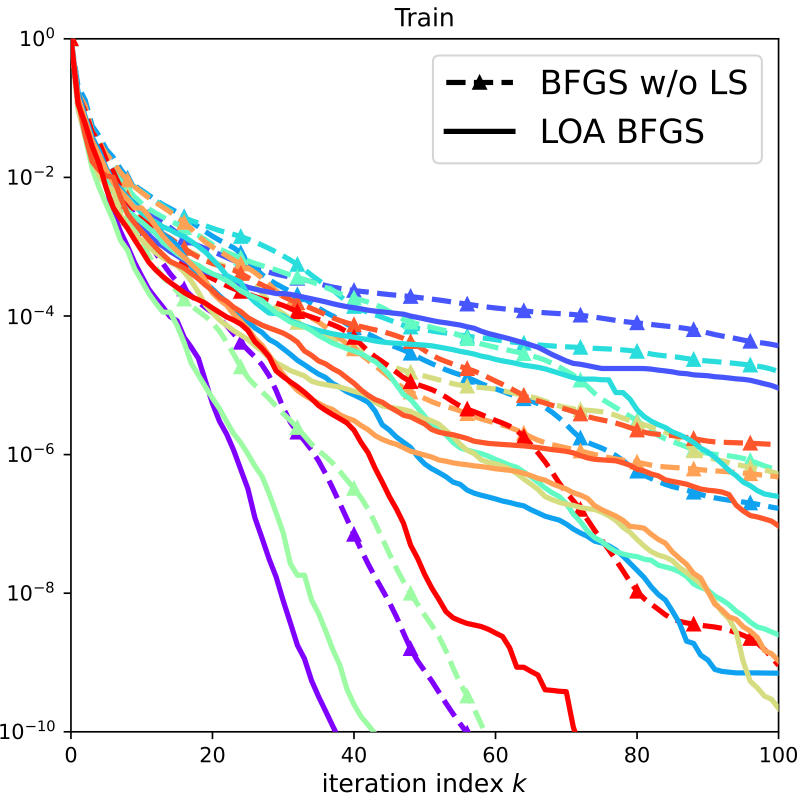}
		& \includegraphics[height=.3\linewidth]{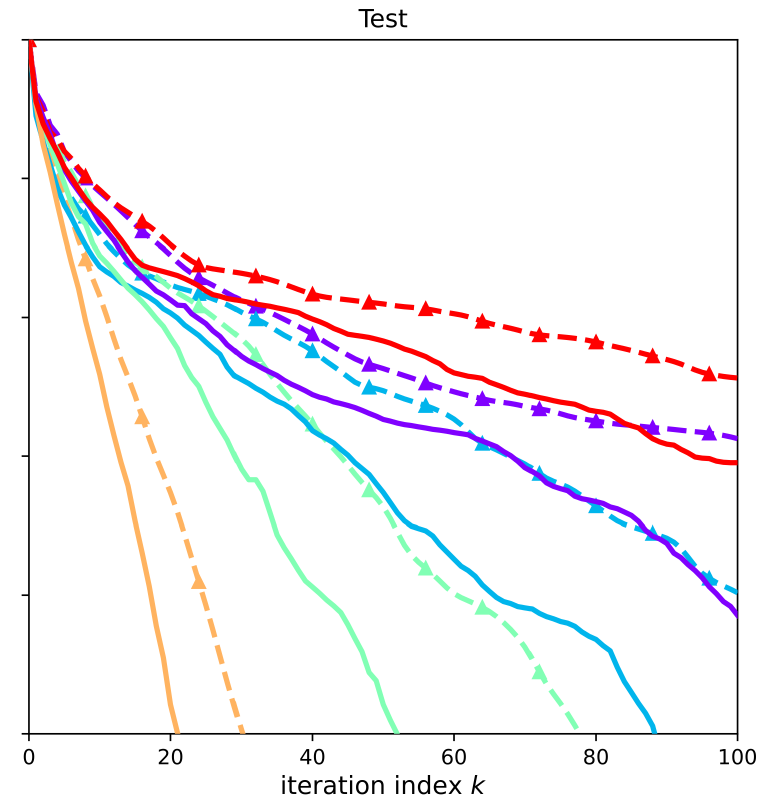}
		\\
		\raisebox{2cm}{\rotatebox{90}{$ \frac{f(x_k)-f^\star}{f(x_0)-f^\star}$}}
		\includegraphics[height=.3\linewidth]{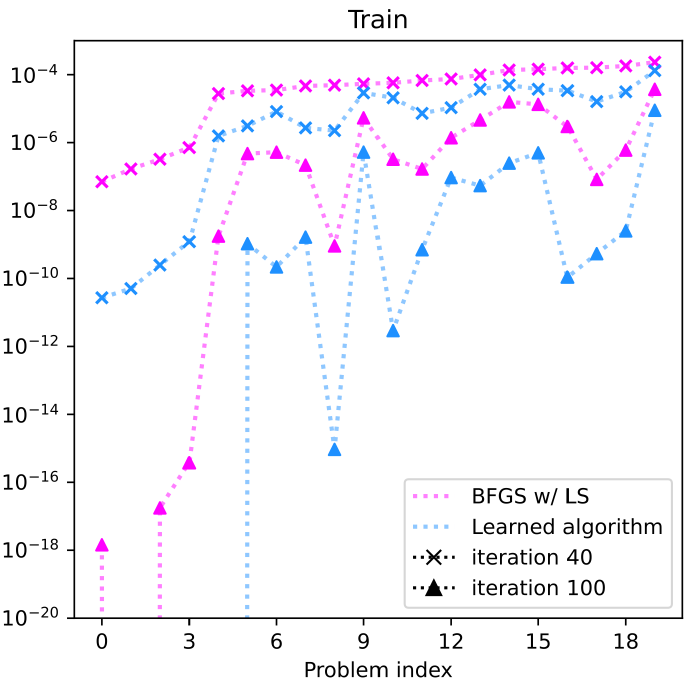}
		& \includegraphics[height=.3\linewidth]{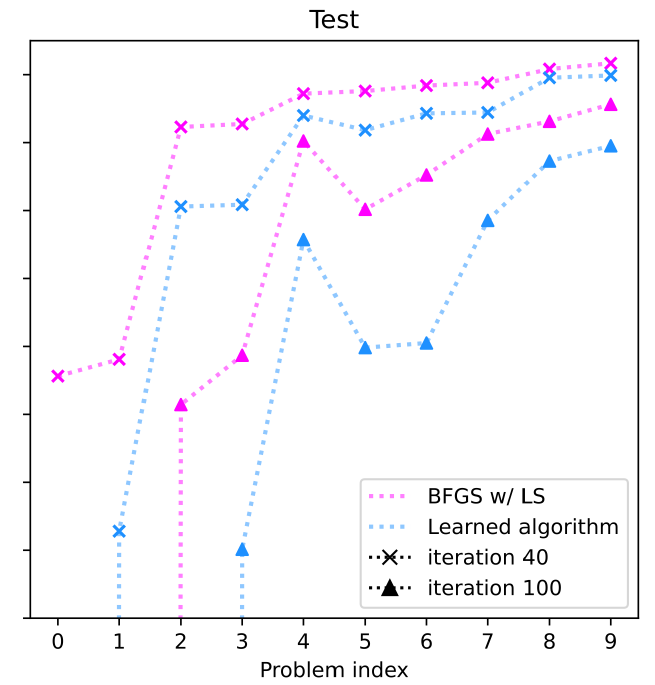}
	\end{tabular}
	\caption{
	Same experiment as Figure~\ref{fig::quad100} but we evaluate the performance of the algorithm used with line-search, \emph{without retraining} it.
		Top row: relative sub-optimality gap against iterations on the training and test sets. Each color represents a different problem. Bottom: relative sub-optimality gap for each problem after $40$ and $100$ iterations. \label{fig::quad100LS}
	}
\end{figure}

\begin{figure}[!h]
	\centering
	\includegraphics[height=.4\linewidth]{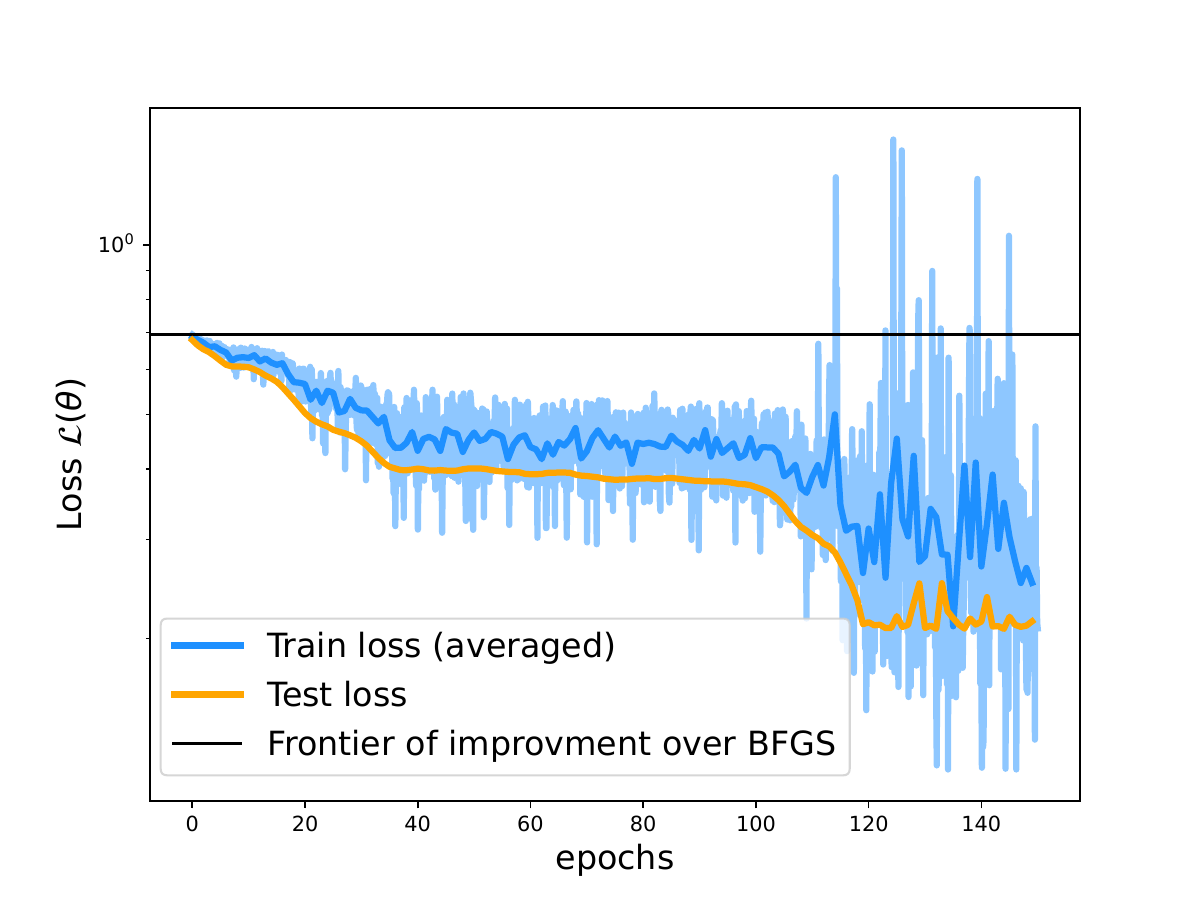}
	\includegraphics[height=.4\linewidth]{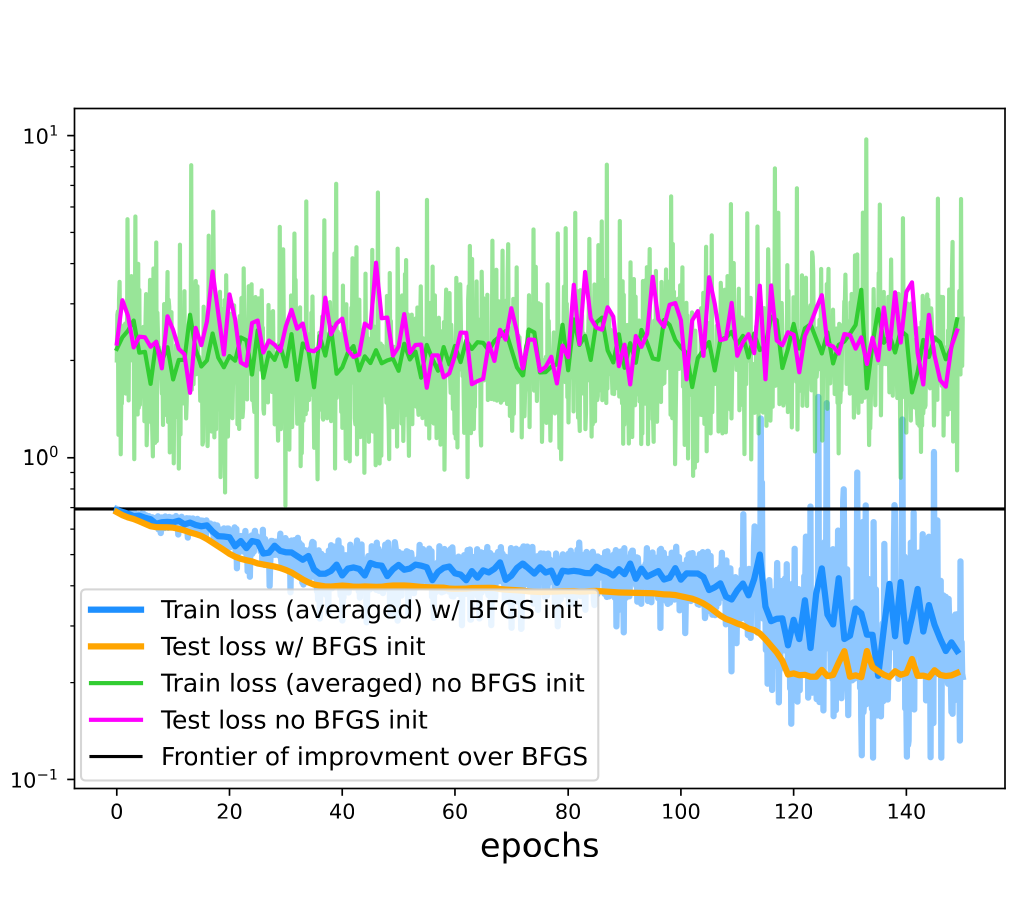}
	\caption{Left: evolution of the training loss and test loss during the training of the model of our algorithm. The blue area shows the value of the stochastic loss and the blue curve represents the average over one epoch. The test loss is computed after each epoch. The black line corresponds to $\log(2)$, any value of $\mathcal{L}(\theta)$ below this line corresponds to an improvement compared to vanilla BFGS.
		Right: Same figure as the left-hand side but comparison with the case where we did not initialize the model to coincide with BFGS, causing instability in training.
		\label{fig::train}}
\end{figure}

\section{ADDITIONAL EXPERIMENTS} \label{app::Exp}

\paragraph{Compatibility with larger-dimensional problems}
Figure~\ref{fig::quad500} shows that our LOA still performs well and consistently outperforms vanilla BFGS in dimension $n=500$ despite having been trained on problems in dimension $n=100$. We note that the improvement is not as dramatic as in dimension $n=100$, yet we managed to transfer good performance in much larger problem than those of the training set, which was our main goal in this experiment.

\paragraph{Compatibility with line-search} Like Newton's method, one usually wants to use QN methods with a step-size $\gamma$ as close as possible to $1$. This may however cause numerical instabilities (\eg in the logistic regression problems). Therefore, QN algorithms, including BFGS are often used with line-search strategy (adapting the step-size based on some rules). It is thus important to evidence that Algorithm~\ref{algo::L2OQN} performs well when used with line-search strategies, \emph{despite having been trained with fixed step-sizes}. The results in Figure~\ref{fig::quad100LS} show that Algorithm~\ref{algo::L2OQN} significantly outperforms BFGS with line-search on almost all problems. This was also the case on the logistic regression problems of Figure~\ref{fig::realworld} where we used line-search for Newton's method, BFGS and our algorithm. Our LOA thus appears to be highly compatible with line-search.

\paragraph{Benefits of our initialization strategy} As mentioned in Section~\ref{sec::practical}, we can easily find a closed-form initialization of the model $\M_\theta$ such that Algorithm~\ref{algo::L2OQN} coincides with BFGS before training. This dramatically stabilizes the training process as shown on Figure~\ref{fig::train} where, without this initialization strategy, the average train loss remains large (despite having tuned the learning rate specifically for that setting). This can be explained by the fact that for a random initialization, the value $f(x_K)$ produced by the algorithm will usually be very large, making it necessary to train with small learning rates, whereas with our strategy we start in a more stable region as evidenced by the smaller oscillations in early training, allowing larger learning rates.

\FloatBarrier


\end{document}